\documentclass[journal, twoside, web]{ieeecolor}  

\usepackage{generic}

\usepackage{amsmath,amssymb}
\usepackage{xcolor}
\usepackage{optidef}
\usepackage[normalem]{ulem}
\usepackage{lipsum}
\usepackage{mathtools}
\usepackage{cuted}
\usepackage{algorithm}
\usepackage{algpseudocode}
\usepackage{breqn}
\usepackage{graphicx}
\usepackage{subfig}
\usepackage{psfrag}
\usepackage[bb=dsserif]{mathalpha}

\newtheorem{remark}{\hspace{0pt}\bf Remark}
\newtheorem{theorem}{\hspace{0pt}\bf Theorem}

\newtheorem{coro}{\hspace{0pt}\bf Corollary}
\newtheorem{prop}{\hspace{0pt}\bf Proposition}
\newcommand\redsout{\bgroup\markoverwith{\textcolor{red}{\rule[0.5ex]{2pt}{0.4pt}}}\ULon}

\DeclareMathOperator*{\argmax}{argmax}
\definecolor{mygreen}{rgb}{0.10,0.50,0.10}

\title{\LARGE \bf
A Bi-Level Optimization Approach to Joint Trajectory Optimization for Redundant Manipulators
}

\author{Jonathan Fried and Santiago Paternain
\thanks{The authors are with the Department of Electrical, Computer, and Systems Engineering, Rensselaer Polytechnic Institute. Email: \{friedj2, paters\}@rpi.edu }
}

\begin{document}

\maketitle

\begin{abstract}

In this work, we present an approach to minimizing the time necessary for the end-effector of a redundant robot manipulator to traverse a Cartesian path by optimizing the trajectory of its joints. Each joint has limits in the ranges of position, velocity and acceleration, the latter making jerks in joint space undesirable. The proposed approach takes this nonlinear optimization problem whose variables are path speed and joint trajectory and reformulates it into a bi-level problem. The lower-level formulation is a convex subproblem that considers a fixed joint trajectory and maximizes path speed while considering all joint velocity and acceleration constraints. Under particular conditions, this subproblem has a closed-form solution. Then, we solve a higher-level subproblem by leveraging the directional derivative of the lower-level value with respect to the joint trajectory parameters. In particular, we use this direction to implement a Primal-Dual method that considers the path accuracy and joint position constraints. We show the efficacy of our proposed approach with simulations and experimental results.

\end{abstract}

\section{Introduction}

For several industrial tasks performed by serial robots, the Cartesian path to be followed by the end-effector generally requires fewer degrees of freedom than what the manipulator possesses. In many applications, such as drilling \cite{yundou2018}, 3D printing \cite{mehmet2019}, soldering \cite{owen2022}, deep rolling \cite{wen2020}, or spray coating \cite{heping2008}, rotation around the attached tool center point is considered a free axis, resulting in a kinematically redundant system \cite{conkur_1997}, \cite{leger2016}. This allows adjustment of the robot pose in ways that optimize multiple metrics or complete multiple objectives. However, redundancy also makes the trajectory optimization problem in joint space complex.

Various methods for solving the problem of kinematic redundancy have been studied in the literature, both from control and motion planning perspectives. From a control standpoint, \cite{jasour2009,jin2023,zhang2023} use Model Predictive Control to minimize metrics such as velocity-norm, acceleration-norm, or tracking error. \cite{hou2012} uses manipulability and multi-preview control to complete trajectory tracking and collision avoidance with objects, 
taking into account the kinematic chain. In \cite{li2020}, redundancy is exploited to circumvent the particular scenario where a joint under repetitive motion and heavy load changes from a normal state to a fault. The approach is to formulate a designed observer for joint state monitoring, which is used to detect joint faults, and a primal-dual neural network solver for the resulting constrained optimization problem. \cite{kemeny2003} uses a parameterized form of the manipulator null-space to solve the redundancy, with a modular approach that allows for the incorporation of singularity avoidance or inertia parameters. Similarly, a null-space-based optimization algorithm is proposed by \cite{zanchettin2012} for redundancy resolution of a dual-arm configuration while taking into account the inverse-kinematic controllers already provided by the robot manufacturers

From a motion planning perspective, \cite{reiter2016,daniel1969,wampler1987,shugen2004,galicki2000,chiacchio1990,bobrow1985,pham2014} use a general inverse kinematics method, separating the redundancy from the trajectory optimization. A particular way of doing inverse kinematics is by using differential kinematics and joint state space decomposition. In this kind of method, the inverse kinematics problem is solved with a second-order differential inverse kinematics approach, choosing the joint torques as a variable that can be integrated to obtain the joint motion. A common characteristic of these works is that the optimization problem results in a non-convex problem due to the non-linear mapping from joint states to Cartesian space. Hence, although the problem's structure differs in these works, they rely on generic non-convex optimization solvers. The latter (e.g., \cite{byrd2000trust,waltz2006interior,gill2019practical}) are powerful in the sense that they can tackle a variety of general nonlinear problems but are unable to exploit the particularities of the path traversal time optimization problem.

To circumvent this issue,  \cite{liu2022} considers deep reinforcement learning. In particular, it proposes a structure combining traditional path planning with deep reinforcement learning. It explores both Cartesian and Joint spaces to find a length-optimal solution that also has energy-optimal inverse kinematics while avoiding obstacles. Another solution is the use of Genetic Algorithms in an attempt to avoid potential local optima and circumvent non-linearities \cite{chaoyang2021}. \cite{wu2016} uses a Genetic Algorithm for joint space trajectory planning, with the additional objective of preventing abrupt change in torque level to smooth angular jerk, or minimize the mean square value of Cartesian jerk, in order to prevent vibrations at the end-effector. In \cite{cheng2020}, a motion planning method based on a beetle antennae search algorithm is proposed to reduce the time taken by genetic algorithms. In particular, it does not use inverse or differential kinematics. While it considers variable velocity joint limits it does not take in consideration joint acceleration limits.

With a different perspective, the work done in \cite{vers2009} only considers the path tracking problem. It exploits the underlying structure to reformulate the proposed problem through an adequate change in decision variables that results in a convex formulation which can be efficiently solved with a second-order cone program-based solution. However, it does not take geometric constraints, redundancy, or the translation of Cartesian space motion into joint space motion.

In this work, we look into this problem from a motion planning perspective and optimize for minimum path traversal time under stringent requirements in terms of accuracy \cite{he2023}. In particular, we propose an algorithm that efficiently provides local optimal solutions by exploiting the structure of the joint motion optimization problem, without recursively using inverse kinematics for the optimization procedure of the joint space motion planning. In particular, we identify a bi-level formulation with a convex low-level subproblem similar to \cite{vers2009} that can be efficiently solved. The solution to this low-level problem can then be used to find local solutions to the high-level problem efficiently. Furthermore, the curve parameterization of our proposed approach allows us to modify the entire joint curve at each iteration of the algorithm, unlike previously mentioned quadratic programming solutions (\cite{he2023}) that optimize the curve from point to point. 

In Section \ref{sec_problem_formulation} we formulate a general joint motion optimization problem for a redundant robot manipulator. Next, we propose an equivalent bi-level formulation and prove the bijection between the set of optimal solution of both problems (Section \ref{sec_bilevel}). Furthermore, we establish the convexity of the low-level subproblem, a key component of the optimization method proposed by this work. It is worth pointing out that the convexity of the low-level problem is independent of the manipulator and the Cartesian path considered. We also discuss special cases where the low-level problem has closed-form solutions. In Section \ref{sec_inner_problem}, we discuss algorithms to solve the low-level (or inner) subproblem, and present methods to exploit its solution to compute directional derivatives of the objective with respect to the high-level subproblem's decision variables. Additionally, for those special cases, this directional derivative turns into a subgradient. The constant path speed case was considered in our prior work \cite{fried2024} and it is presented here for completeness. In Section \ref{sec_outer}, 
we describe how the directional derivatives obtained through the lower-level problem can be leveraged in a primal-dual method to optimize the high-level, or outer, subproblem. Section \ref{sec_numerical} presents numerical simulations inspired by a cold spraying application and comparison against general non-linear solvers, while Section \ref{sec_exp} presents experimental results obtained in a robot manipulator with comparisons between simulated and experimental results. This is followed by concluding remarks (Section \ref{sec_conclusion}).

\section{Problem Statement}\label{sec_problem_formulation}

{Consider the problem of optimizing the joint motion of a redundant robot manipulator with $n$ joints. Let $s_i$, where $s_0=0$ and $s_N=1$ with $i=0,\ldots,N$ denote the path length and 
$\chi^d_i \in \mathbb{R}^m$ with $i=0,\ldots N$ the points in the Cartesian path that the robot's end-effector must trace. Let $q_i \in \mathbb{R}^{n}$  for $i=0,\ldots N$, represent the manipulator's path in joint space. Since the robot manipulator is redundant, it has more degrees-of-freedom than required to trace $\chi^d$, i.e.,  $n>m$. Thus, multiple trajectories in joint space result in the same Cartesian path. Among these, we are interested in finding one with the minimum traversal time.

The problem of optimizing the joint path is equivalent to choosing $q$ that maximizes path velocity $\dot{s}$ or to minimize $t_f$, the travel time of the manipulator across the path $\chi^d$. Let $t_i$ with $i=0,\ldots N-1$ denote the time interval needed to go from path point $s_i$ to path point $s_{i+1}$. Then, we can write $t_f$ and its first order approximation as 
\begin{equation}
\label{eq:objective}
t_f = \sum_{i=0}^{N-1} t_i \approx \sum_{i=0}^{N-1} \frac{\Delta s_i}{\dot{s}_i},
\end{equation}
 where $\Delta s_i=s_{i+1}-s_i$.
Note that the objective does not depend on $\dot{s}_N$, since the path is completed at point $N$. Thus, the path speed at that point does not contribute to the total time. 

Each joint $j=1,\ldots,n$ of the manipulator operates under position, velocity, and acceleration constraints, given by 
\begin{equation}
    \underline{q}_j \leq q_{ij} \leq \overline{q}_j,\;\dot{\underline{q}}_j \leq \dot{q}_{ij} \leq \dot{\overline{q}}_j\;\mbox{and}\;    \ddot{\underline{q}}_j \leq \ddot{q}_{ij} \leq \ddot{\overline{q}}_j,
    \label{eq:joint_restrictions}
\end{equation}
where $q_{ij}$ is the position of joint $j$ at path point $s_i$, $\forall \, j =\{1,\ldots,n\}$ and $\forall \,i =\{0,\ldots,N\}$. We approximate the joints' trajectory $q_{ij}$ linearly with respect to vectors of parameters $\theta_j \in \mathbb{R}^d$ for all $j=1,\ldots{n}$ so that 
\begin{equation}
    q_{ij} = p(s_i)\theta_j,
    \label{eq:paramterization}
\end{equation}
where $p(s_i) \in \mathbb{R}^{1\times d}$ is a twice differentiable vector basis in $s$. This assumption is common in the literature (see e.g., \cite{koubiaspline,zhu2022,embry2018}). Applying the chain rule, the first derivative of the joint position with respect to the time yields
\begin{equation}\label{q_dot}
\dot{q}_{ij} = q_{ij}^{\prime}(s_i)\dot{s}_i = p^{\prime}(s_i)\theta_j \dot{s}_i.
\end{equation}
Likewise, its second derivative is given by
\begin{equation}
\label{eq:sec_der}
\ddot{q}_j = q_{ij}^{\prime\prime}\dot{s}_i^2+q_{ij}^{\prime}\ddot{s}_i = p^{\prime\prime}(s_i)\theta_j \dot{s}_i^2+p^{\prime}(s_i)\theta_j\ddot{s}_i.
\end{equation}
We define for simplicity the following notation $p_i\!\!=\!\!p(\!s_i\!), p^\prime_i\!\!=\!\!p^\prime\!(\!s_i\!), p^{\prime\prime}_i\! =\!p^{\prime\prime}\!(\!s_i\!)$, which we will use for the remainder of this work. 
Note that, the path velocities $\dot{s}_i$ and accelerations $\ddot{s}_i$ are not independent. Assuming that between points in the path the acceleration is constant, we can write the velocity as 
\begin{equation}
\label{eq:speed_dynamic}
    \dot{s}_{i+1} = \dot{s}_i + \ddot{s}_i t_i.
\end{equation}
 Under the same assumptions the path dynamics take the form
\begin{equation}
\label{eq:path_dynamic}
    s_{i+1} = s_i + \dot{s}_i t_i + \ddot{s}_i \frac{t_i^2}{2}.
\end{equation}

For a given joint pose $q_i$, the Cartesian point $\chi_i$ of the end-effector is a direct result of its kinematics. Let $k:\mathbb{R}^n\to\mathbb{R}^m$ represent the forward kinematics of the manipulator that map $q$ to $\chi$. 
Then, the robot manipulator is subject to the following kinematic constraints \cite{siciliano2007}
\begin{equation}
    \chi_i = k(q_i),\;\;\forall i=0,\ldots,N,
    \label{eq:fwd_kin}
\end{equation}
where, $q_i$ is a vector containing all the joints of the manipulator. Formally, $q_i\!=\!\left(\!I_n\otimes p_i\!\right)\theta$, where $\otimes$ denotes the Kronecker product, $I_n$ is a $n\times n$ identity matrix and $\theta = \left[\theta_1^T,\ldots,\theta_n^T\right]^T$. In what follows, we omit the parenthesis since, given the dimensions of the matrices, this order of operations is the only possible one. Then, the error between the desired and the actual Cartesian path is given by
\begin{equation}
    \label{eq:error}
    E(\theta) = \left(\sum_{i=0}^N \left\|\chi^d_i-k(I_n\otimes p_i\theta) \right\|_2^p\right)^{\frac{1}{p}},\end{equation}
where $E(\theta)$ is the $p-$norm of the path error. 
 With this, we are now able to formalize the joint and speed optimization problem that minimizes the travel time $t_f$

\begin{mini}|s|
  {\stackrel{t,\dot{s}\in \mathbb{R}_+^N,}{\stackrel{\ddot{s}\in\mathbb{R}^{N-1},}{\stackrel{\theta \in \mathbb{R}^{nd}}{}}
  }
  }
  { \sum_{i=0}^{N-1} \frac{\Delta s_i}{\dot{s}_i}}{\label{opt:generaldiscrete}}{{t_f}^\star=}
   \addConstraint{}{}{E(\theta)\leq \epsilon}
 \addConstraint{}{}{\underline{q}_j \leq p_i\theta_j\leq \overline{q}_{j}}  
  \addConstraint{}{}{\dot{\underline{q}}_j \leq p_i^\prime\theta_j\dot{s}_i\leq \dot{\overline{q}}_{j}}
  \addConstraint{}{}{\ddot{\underline{q}}_{j} \leq p^{\prime\prime}_i\theta_j \dot{s}^2_i + p^{\prime}_i\theta_j \ddot{s}_{i}\leq \ddot{\overline{q}}_{j}}
  \addConstraint{}{}{\dot{s}_{i+1} = \dot{s}_i + \ddot{s}_i t_i}
  \addConstraint{}{}{s_{i+1} = s_i + \dot{s}_i t_i + \ddot{s}_i \frac{t_i^2}{2}}
 \end{mini}
%
%
where $E(\theta)$ is the error on the Cartesian as defined in \eqref{eq:error} and $\epsilon$ a tolerance on it. In the previous expressions we have omitted that the constraints need to hold for all $i=\left\{0,\ldots, N\right\}$ and $j\! =\! \left\{1,\ldots,n\right\}$.

The problem at hand has two distinct type sources of non-convexity, the nonlinear relationship between joint and Cartesian spaces (which ultimately results in the function $E(\theta)$ being nonconvex), and the kinematic and dynamic constraints involving several products of the different decision variables. Although powerful generic non-convex optimization tools exist~\cite{byrd2000trust,waltz2006interior,gill2019practical}, in general they take a long time to converge. 

In this work, we propose to exploit the structure of the problem to develop efficient algorithms that obtain comparable results to the non-convex solvers in less computational time (see Section \ref{sec_comparison}). The structure to be exploited is that the Cartesian path constraints and the joint limits depend only on $\theta$ and do not depend on the kinematic and dynamic variables $\dot{s}_i,\ddot{s}_i,t_i$. Indeed, they depend on the position of the manipulator's joints but not on its velocity. On the other hand, the dynamic and kinematic constraints depend on both decision variables. The interpretation is that that velocity and acceleration constraints depend not only on how fast the path is traced but also on the slopes and curvature each joint will need to perform. 

This observation suggests that with an adequate change of decision variables, similar to \cite{vers2009}, we can obtain a bi-level structure in which the lower-level solves the path velocity $\dot{s}_i$ and the higher-level contains the path and joint position constraints of the original problem. Notably, the inner problem is convex. This guarantee is also independent of the complexity of the task at hand, i.e., the path that needs to be traced and the characteristics of the manipulator. Formalizing this formulation and its properties is the subject of the next section.

\section{A Bi-level Formulation with Convex Low-Level Problem}\label{sec_bilevel}

In this section, we propose a bi-level problem with a convex low-level subproblem with the same solution as \eqref{opt:generaldiscrete}. We start by defining the lower-level problem

\begin{mini}|s|
  {\dot{s}_i^2\geq 0}{ \left(\sum_{i=0}^{N-1} \frac{\Delta s_i}{\sqrt{\dot{s}^2_i}}\right)^2}{\label{opt:inner}}{V(\theta)=}
  \addConstraint{-\dot{\underline{q}}_{j}^2}{\leq \mbox{sign}(p_i^\prime\theta_j)(p_i^\prime\theta_j)^2\dot{s}_i^2 \leq \dot{\overline{q}}_{j}^2}{}
\addConstraint{\ddot{\underline{q}}_{j} }{\leq \rho_{ij}\dot{s}_i^2 + \frac{p^{\prime}_i\theta_j}{2\Delta s_i} \dot{s}^2_{i+1}\leq \ddot{\overline{q}}_{j}}{},
\end{mini}
where $\rho_{ij} = \left(p^{\prime\prime}_i\theta_j - p^{\prime}_i\theta_j/2\Delta s_i\right)$. As in \eqref{opt:generaldiscrete} the constraints need to be satisfied for all $j=1,\ldots,n$ and for all $i=0,\ldots,N$. Note that $V(\theta)$ is the square of the optimal traverse time for a given trajectory defined by $\theta$. 

Then, let us define the upper-level problem 
\begin{mini}|s|
  {\theta \in \mathbb{R}^{nd}}{ V(\theta)}{\label{opt:outer}}{({t_f}^\star)^2=}
  \addConstraint{E(\theta)}{ \leq \epsilon}{}
  \addConstraint{\underline{q}_j }{\leq p_i\theta_j\leq \overline{q}_{j}}{}  \end{mini}
where in the previous expression the second constraint needs to hold for $i=0,\ldots, N$ and $j=1,\ldots,n$. 

The structure proposed here follows the intuition discussed in Section \ref{sec_problem_formulation}, where the upper-level problem is independent of the kinematic and dynamic variables and the lower-level problem minimizes the traversal time given a trajectory to respect the manipulator's limits.

It is worth pointing out that the proposed formulation is independent of the task at hand, i.e., the characteristics of the manipulator and the geometry of the curve to be traced. Indeed, the lower-level problem applies to any optimal time traversal time and imposes generic kinematic and dynamic constraints on the joints of the manipulator. 

In addition, observe that instead of solving for $t_f$ we consider $t_f^2$ as the objective. This modification does not affect the solution of the problem, as both time and path velocity are positive, and it makes the function $V(\theta)$ smoother and in some cases convex (See Theorems \ref{thm:inner}, \ref{thm:inner_cte}, \ref{thm:inner_noacc} and Remark \ref{rmk:convex}).

Note that we are setting the decision variable as $\dot{s}_i^2$ instead of $\dot{s}_i$. Likewise, the times $t_i$ and accelerations $\ddot{s}_i$ have been removed. This allows for a reduction in the number of variables and constraints. Furthermore, the lower-level problem is convex. This is the subject of the next proposition. 
\begin{prop}\label{prop_convex_problem}
    Problem \eqref{opt:inner} is a convex optimization problem for all $\theta\in\mathbb{R}^{nd}$. Moreover, its solution is unique.
\end{prop}
\begin{proof}
   Notice that all the constraints are linear when the variable considered is $\dot{s}_i^2$. Hence, we are left to show that the objective in \eqref{opt:inner} is a convex function. Since $y^2$ is a convex and non-decreasing function for $y\geq 0$, to establish the convexity of the objective it suffices to show the convexity of $f: \mathbb{R}^{N}_+ \to \mathbb{R}_+$, with $f$ defined as follows
   \begin{equation}
       f(x) = \sum_{i=1}^{N} \frac{a_i}{\sqrt{x_i}},
   \end{equation}
for $a_i>0$ for all $i=1,\ldots, N$. This fact follows from the composition of convex functions~\cite[Section 3.2.4]{boyd2004convex}. First note that the function $f(x)$ is indeed non-negative. To establish its convexity, note that the Hessian of $f$ is a diagonal matrix with 
\begin{equation}
    (\nabla^2 f(x))_{ii} = \frac{3 a_i}{4} x_i^{-5/2}\; \forall i=1,\ldots N.
\end{equation}
   Since $x_i>0$, the Hessian is positive definite. Thus, $f(x)$ is convex. The uniqueness of the solution follows from the fact that the objective is strictly convex. We argue by contradiction. Assume that there exist two optimal solutions $x^\star_1,x^\star_2$ such that $x^\star_1 \neq x^\star_2$.  It follows that for any $\alpha \in[0,1]$, $\alpha x^\star_1+(1-\alpha)x^\star_2$ is feasible since the set of optimal solutions is convex. Furthermore, since $f(x)$ is strictly convex it holds that
   \begin{equation}
       f(\alpha x^\star_1+(1-\alpha) x^\star_2) < \alpha f(x^\star_1) +(1-\alpha) f(x^\star_2) = f(x^\star_1),
   \end{equation}
   where the last equality follows from the fact that $x^\star_1$ and $x^\star_2$ are optimal and thus they have the same value. Note that the above expression establishes that $\alpha x^\star_1+(1-\alpha)x^\star_2$ has a lower value than the optimal, which is a contradiction. This completes the proof of the uniqueness of the solution.
\end{proof}

The main advantage of this proposed separation is that the inner subproblem can be solved efficiently through convex optimization tools e.g., CVX (\cite{cvx},\cite{gb08}), YALMIP \cite{Lofberg2004}, or CDSP \cite{borchers1999}.  Under some particular conditions, the inner problem can be solved even in closed form, and its value is reduced to a convex function depending only on $\theta$ (Section \ref{sec_inner_problem}). Thus, resulting in further efficiency in the computation of the solution. For the proposed formulation to be useful, it needs to solve problem \eqref{opt:generaldiscrete}. This is the subject of the next proposition.
\begin{prop}
{There exists a bijection between the optimal solutions to the bi-level problem \eqref{opt:inner}--\eqref{opt:outer} and the optimal solutions to \eqref{opt:generaldiscrete}. } 
\end{prop}
\begin{proof}
Let $(\dot{s}^G,\ddot{s}^G,t^G,\theta^G)$ be a feasible solution to \eqref{opt:generaldiscrete}. We will establish that $(\dot{s}^G,\theta^G)$ is feasible for \eqref{opt:inner}--\eqref{opt:outer}. Since the constraints in \eqref{opt:outer} do not depend on $\dot{s}$ and take the same form as the first two constraints in \eqref{opt:generaldiscrete} it follows that $(\dot{s}^G,\theta^G)$ is feasible for \eqref{opt:outer}. To prove that $(\dot{s}^G,\theta^G)$ is feasible for \eqref{opt:inner} it suffices to show that 
\begin{equation}
\label{eq:speed_dynamic}
    \ddot{s}^G_{i} = \frac{(\dot{s}^G_{i+1})^2 - (\dot{s}^G_{i})^2}{2\Delta s_{i}^G}.
\end{equation}
If the previous claim holds, then the second constraint in \eqref{opt:inner} holds for $(\dot{s}^G,\theta^G)$. The first constraint of \eqref{opt:inner} is satisfied by $(\dot{s}^G,\theta^G)$ trivially since the constraint is the same as the third constraint in \eqref{opt:generaldiscrete}. 

We turn our attention to establishing \eqref{eq:speed_dynamic}. We will analyze the claim in two particular cases. (i) Assume that $\ddot{s}_i^G = 0$, then it follows from the fifth constraint in \eqref{opt:generaldiscrete} that $\dot{s}^G_i= \dot{s}^G_{i+1}$. Thus proving \eqref{eq:speed_dynamic}. (ii) If $\ddot{s}_i^G \neq 0$, it follows from the same constraint that 
\begin{equation}
    t_i^G = \frac{\dot{s}_{i+1}^G-\dot{s}_i^G}{\ddot{s}^G_i}.
\end{equation}
Substituting the previous expression in \eqref{eq:path_dynamic} yields \eqref{eq:speed_dynamic}.

Using a similar argument, one can establish that for any solution to the bi-level problem, $(\dot{s}^B,\theta^B)$, there exists $(\ddot{s}^\dagger,t^\dagger)$ so that $(\dot{s}^B,\ddot{s}^\dagger,t^\dagger,\theta^B)$ is a solution to \eqref{opt:generaldiscrete}.

To complete the proof that the two problems are equivalent it suffices to notice that the ordering imposed by the objectives in \eqref{opt:inner}--\eqref{opt:outer} and \eqref{opt:generaldiscrete} is the same. Indeed, the objective in the bi-level formulation is the composition of a monotonic function with the objective function in \eqref{opt:generaldiscrete}.

\end{proof}
Having confirmed that the bi-level formulation recovers the solutions to the problem \eqref{opt:generaldiscrete}, we can take advantage of this structure to solve the problem at hand. This is the subject of the next section, where we exploit the solutions to the inner problem to compute subgradients or gradients of $V(\theta)$ which are used in Section \ref{sec_outer} to solve \eqref{opt:outer}. {Before doing so, in Sections \ref{subsec_problem_formulation_cv} and \ref{subsec_problem_formulation_noacc} we discuss two particular cases of the problem where the solution of the inner problem can be computed in closed form.}

\subsection{Constant Path Velocity Case}\label{subsec_problem_formulation_cv}

Consider a case with another restriction to its dynamics, where the path must also be traversed at constant velocity $\dot{s}$, i.e, $\ddot{s}_i = 0$. This restriction is common for some industrial activities, such as soldering, spray coating, or 3D printing. 

 In this case, optimizing the joint path is equivalent to choosing $q$  that achieves the maximum constant path velocity (constant $\dot{s}$) to minimize $t_f>0$. Particularly, we can reduce equation \eqref{eq:sec_der} into the following expression
\begin{equation}\label{q_ddot_cte}
\ddot{q}_{ij} = p^{\prime\prime}_i\theta_j \dot{s}^2.
\end{equation}
Substituting \eqref{q_dot} and \eqref{q_ddot_cte} in \eqref{eq:joint_restrictions} yields the following constraints
\begin{equation}\label{eq:param_vel_restriction_cte}
    \dot{\underline{q}}_j \leq  p'_i\theta_j\dot{s} \leq \dot{\overline{q}}_{j},\; \ddot{\underline{q}}_{j} \leq  p''_i\theta_j\dot{s}^2  \leq \ddot{\overline{q}}_{j}.
\end{equation}

Additionally, the objective is also simplified. As $s_0 = 0$ and $s_N = 1$, if $\dot{s}$ is constant, we can express $t_f$ as

\begin{equation}
    \label{eq:objective_constant}
    t_f = \frac{1}{\dot{s}}
\end{equation}

In this scenario the lower-level problem \eqref{opt:inner} reduces to
\begin{mini}|s|
 {\dot{s}^2\geq 0}{\frac{1}{\dot{s}^2}}{\label{opt:inner_cte}}{V(\theta) =}
  \addConstraint{-\dot{\underline{q}}_{j}^2}{\leq \mbox{sign}(p_i^\prime\theta_j)(p_i^\prime\theta_j)^2\dot{s}^2 \leq \dot{\overline{q}}_{j}^2\;}{}
 \addConstraint{\ddot{\underline{q}}_{j} }{\leq p''_i\theta_j\dot{s}^2 \leq \ddot{\overline{q}}_{j},\;}{} 
\end{mini}
where the constraints need to hold for all $i = 0, \dots, N$ and $j = 1, \dots, n$. 

\subsection{No Acceleration Constraints Case}\label{subsec_problem_formulation_noacc}

{Consider a case where we can disregard the acceleration constraints of the robot joints, i.e. the robot manipulator will typically reach its speed limit first for any continuous path in joint space. This particular case is reasonable for industrial robots under a light load and with fast internal controllers (\cite{zanchettin2011,giles2021,liu2023}). 

 In this case, optimizing the joint path is equivalent to choosing $q$  that achieves maximum velocity for each path point $i$ to minimize $t_f>0$. Particularly, we can remove the constraint given by equation \eqref{eq:sec_der}, reducing Problem \eqref{opt:generaldiscrete} into
\begin{mini}|s|
  {\dot{s}_i\geq 0}{ \left(\sum_{i=0}^{N-1} \frac{\Delta s_i}{\sqrt{\dot{s}^2_i}}\right)^2}{\label{opt:inner_noacc}}{\!\!\!\!\!\!\!\!\!\!\!\!\!\!\!\!\!\!\!\!\!\!V(\theta)=}
  \addConstraint{-\dot{\underline{q}}_{j}^2}{\leq \mbox{sign}(p_i^\prime\theta_j)(p_i^\prime\theta_j)^2\dot{s}_i^2 \leq \dot{\overline{q}}_{j}^2,}{}
\end{mini}
where the constraint need to hold for all $i = 0, \dots, N$ and $j = 1, \dots, n$

In the following sections, we discuss how to solve the inner and outer subproblems. First, we discuss the inner subproblem, its solution, special cases, and how to find a descent direction of its value $V(\theta)$ with respect to $\theta$. Afterwards, we discuss the outer subproblem and how to use the results of the previously computed inner subproblem to improve the path time $t_f$.

\section{Inner Subproblem Optimization}\label{sec_inner_problem}

This problem can be solved efficiently since it is a convex optimization problem (see Proposition \ref{prop_convex_problem}). In addition, by solving it we also remove from the outerproblem the $2\times n \times(2N+1)$ constraints that correspond to velocity and acceleration. Furthermore, we show that the solution of the low-level problem can be used to efficiently compute the derivative of $V(\theta)$ with respect to $\theta$, which is needed to optimize the higher problem. These aspects result in an improvement in terms of the performance of our algorithm in comparison with standard general purpose solvers (see Section \ref{sec_numerical}). We structure this section in three parts discussing each one of the cases in Section \ref{sec_bilevel}.

\subsection{General Case}

We begin this section by rewriting~\eqref{opt:inner} in a compact form that emphasizes the linearity of the constraints with respect to $\dot{s_i}^2$. To do so, we start by defining the vectors $\dot{s}^2=[\dot{s}_0^2,\dots,\dot{s}_N^2]^T$, $q_\lim = [[\dot{\overline{q}}_{1}^2,\dot{\overline{q}}_{2}^2,$ $\dots,\dot{\overline{q}}_{n}^2,\dot{\underline{q}}_{1}^2,\dots,\dot{\underline{q}}_{n}^2] \otimes \otimes \mathbb{1}^{1\times (N+1)}$, $[\ddot{\overline{q}}_{1},\ddot{\overline{q}}_{2},\dots,$ $\ddot{\overline{q}}_{n},$ $-\ddot{\underline{q}}_{1},$ $\dots,$ $-\ddot{\underline{q}}_{n}] \otimes \mathbb{1}^{1\times N}]^T$ $\in \mathbb{R}^{(2n(2N+1))}$.

Let $A_{sj} \in \mathbb{R}^{N+1 \times N+1}$, with $j=1,\ldots, n$ be matrices containing the velocity constraints for each joint. These are diagonal matrices whose $i$-th elements on the diagonal take the form $\mbox{sign}(p_i'\theta_j)(p_i'\theta_j)^2$. With these definitions, let $A_{spd} = \left[A_{s1},A_{s2},\cdots,A_{sn}\right]^T$ be a matrix containing the velocity constraints for all joints. Similarly, define acceleration matrices for each joint $j=1,\cdots,n$ as $A_{aj} \in \mathbb{R}^{N \times (N+1)}$ where for the main diagonal $a_{aj_{i,i}} = p^{\prime\prime}_i\theta_j - \frac{p^{\prime}_i\theta_j}{2\Delta s_i}$, for the upper diagonal $a_{aj_{i,i+1}} = \frac{p^{\prime}_i\theta_j}{2\Delta s_i}$, and $0$ otherwise, $\forall i = 0,\cdots,N-1$.}. With these definitions, let $A_{acc} = \left[A_{a1}^T,A_{a2}^T,\cdots,A_{an}^T\right]^T$ be a matrix containing the acceleration constraints for all joints. 
Let us then, define 
\begin{equation}
    A(\theta) = \left[\begin{array}{c}
        A_{spd} \\
         -A_{spd} \\
         A_{acc} \\
         -A_{acc}
    \end{array}\right],
\end{equation}
and write compactly the lower level problem \eqref{opt:inner} as
\begin{mini}|s|
  {\dot{s}^2\in \mathbb{R}_+^{N+1}}{ \left(\sum_{i=0}^{N-1} \frac{\Delta s_i}{\sqrt{\dot{s}^2_i}}\right)^2}{\label{opt:inner_lin}}{V(\theta)=}
  \addConstraint{A(\theta)}{\dot{s}^2 \leq q_\lim}{}.
\end{mini}
Recall that the constraints of the above problem represent joint velocities and accelerations. In particular, note that if the joint parameters are such that the position is constant for all joints, any $\dot{s}^2 \geq 0$ is feasible. Thus the infimum $V(\theta) = 0$ is attained by selecting $\dot{s}^2$.

{We next focus on deriving an expression for the derivative of $V(\theta)$. To do so, define the Lagrangian associated with~\eqref{opt:inner_lin} }

\begin{multline}
    \label{eq:inner_lag}
     \mathcal{L}_I (\theta,\dot{s}^2,\zeta) = \left(\sum_{i=0}^{N-1} \frac{\Delta s_i}{\sqrt{\dot{s}_i^2}}\right)^2 + \zeta(A(\theta)\dot{s}^2 - q_\lim),
\end{multline}
where $\zeta = [\zeta_1, \zeta_2, \cdots, \zeta_{2n(2N+1)}]\in \mathbb{R}^{2n(2N+1)}_+$ and $ \mathcal{Z}^\star$ is the solution set of the dual problem, i.e., 
\begin{equation}
    \mathcal{Z}^\star = \argmax\limits_{\zeta \geq 0} \min\limits_{\dot{s}^2} \mathcal{L}_I (\theta,\dot{s}^2,\zeta).
    \end{equation} 
    Since the inner problem is convex (see Proposition \ref{prop_convex_problem}) it follows that 
\begin{equation}
    V(\theta) = \mathcal{L}_I(\theta, (\dot{s}^2)^\star, \zeta^\star), 
\end{equation}
where $(\dot{s}^2)^\star$ is the unique solution to \eqref{opt:inner_lin} (see Proposition \ref{prop_convex_problem}) and $\zeta^\star \in \mathcal{Z}^\star$. One could then attempt to decrease the value of $\theta$ by computing the gradient of the above expression with respect to $\theta$ and updating $\theta$ in the negative direction of that gradient. However, it is well known that if the dual solution is not unique the function $V(\theta)$ is not differentiable~\cite[Ch. 6]{bazaraa2006nonlinear}. Yet, we can guarantee that this direction resembles a subgradient. The precise meaning of this sentence is the subject of the following Theorem.

\begin{theorem}\label{thm:inner}

Consider the optimization problem \eqref{opt:inner} and let $((\dot{s}^{2})^\star,\zeta^\star)$ be a primal dual solution. Then, the following result holds
%
\begin{equation}\label{eq:thm1}
  \frac{\partial \mathcal{L}_I(\theta,(\dot{s}^{2})^\star,\zeta^*)}{\partial \theta}^T \theta \geq 0.  
\end{equation}
%
where $\mathcal{L}_I$ is the Lagrangian defined in \eqref{eq:inner_lag}.

\end{theorem}
    \begin{proof}
        See Appendix \ref{app:PF0}. 
    \end{proof}

The significance of the previous result is that the direction $\frac{\partial \mathcal{L}_I(\theta,(\dot{s}^{2})^\star,\zeta^*)}{\partial \theta}$ can be used to update $\theta$ to minimize $V(\theta)$ akin a subgradient descent method. Indeed, since $\theta =0$ belongs to the minimizers of $V(\theta)$, updating $\theta$ in the negative of this direction (with a sufficiently small step size) reduces the distance to the optimal set. In fact, for the particular cases of constant velocity and no acceleration constraints, this direction is the subgradient. We will prove this claim by showing that $V(\theta)$ is convex. This is the subject of the following two subsections.

\subsection{Constant Path Velocity Case}\label{sec_inner_cte}
In this Section, we consider the case of constant path velocity~\eqref{opt:inner_cte}, which as previously mentioned has a closed form and its value $V(\theta)$ results in a convex function.  This is the subject of the next Theorem.

{\color{black}{

\begin{theorem}\label{thm:inner_cte}
Consider the optimization problem \eqref{opt:inner_cte} parameterized by $\theta$. If the velocity and acceleration joint limits satisfy $\dot{\underline{q}}_{j}, \ddot{\underline{q}}_{j} < 0$, $\dot{\overline{q}}_{j}, \ddot{\overline{q}}_{j} > 0$, then $V(\theta)$, the value of the optimization problem \eqref{opt:inner}, is given by

\begin{dmath}\label{eq_solution_inner_cte}
    V(\theta) \!\!=\!\!\!\!\!\! \max_{\substack{j \in\{1,\dots,n\} \\ i\in\{0,\ldots,N\}}} \!\!\!\left\{\!\!\left(\!\!\max\left\{\frac{p'_i\theta_j}{\dot{\overline{q}}_{j}},\frac{p^{\prime}_i\theta_j}{\dot{\underline{q}}_{j}}\right\}\!\!\right)^2\!\!\!\!\!,  \max\!\!\left\{\frac{p''_i\theta_j}{\ddot{\overline{q}}_{j}},\frac{p^{\prime\prime}_i\theta_j}{\ddot{\underline{q}}_{j}}\right\}\!\!\!\right\}.
\end{dmath}

Furthermore, $V(\theta)$ is convex and $(\dot{s}^2)^\star=1/V(\theta)$.

\end{theorem}
\begin{proof}
    We defer the proof of \eqref{eq_solution_inner_cte} to Appendix \ref{app:PF1}. The convexity of the value of the inner subproblem $V(\theta)$ follows from \eqref{eq_solution_inner_cte} and the fact that the pointwise maximum of convex functions is convex~\cite{rockafellar1969convex}. $(\dot{s}^2)^\star = 1/V(\theta)$ follows from the definition of the problem.
\end{proof}

The joint limits hypothesis of Theorem \ref{thm:inner} implies that the joints can move in both directions (clockwise/counterclockwise for a revolute joint, forward/backward for a prismatic joint). Note that this choice also guarantees that $\dot{s}=0$ is feasible. This implies that if the robot does not move, its joints respect their velocity and acceleration limits. These are reasonable assumptions for any manipulator.For most industrial robots, the joint velocity and acceleration limits have the same magnitude regardless of the direction in which they are moving. In this case, \eqref{eq_solution_inner_cte} results in the simpler form presented in Corollary \ref{cor:inner_cte}. 

The closed form solution in \eqref{eq_solution_inner_cte} allows us to efficiently compute the solution of the inner problem \eqref{opt:inner_cte}. Note that the solution to the problem can be computed by evaluating the $2 \times n \times (2N+1)$ functions in \eqref{eq_solution_inner_cte} (corresponding to the constraints of the inner problem). Intuitively, by by finding the tightest constraint, one can find the largest allowed speed. This evaluation is empirically faster than solving the optimization problem using, for instance, CVX (see Section \ref{sec_numerical}). Furthermore, we can compute the subgradient of $V(\theta)$ in closed form. In Corollary \ref{cor:inner_cte}, we provide its expression for the simplified hypothesis of joints with symmetric limits.

\begin{coro}
    \label{cor:inner_cte}
    Under the additional hypothesis that the joint velocity and acceleration limits satisfy
    $    -\dot{\underline{q}} = \dot{\overline{q}}$ and $-\ddot{\underline{q}} = \ddot{\overline{q}},$
    %
    the value $V(\theta)$ defined in \eqref{opt:inner_cte} reduces to 
    \begin{dmath}\label{eq_solution_inner_cor}
    V(\theta) = \max_{\substack{j \in\{1,\dots,n\} \\ i\in\{0,\ldots,N\}}} \left\{f^1_{ij},  f^2_{ij}\right\},
\end{dmath}
    where
    \begin{align}
        f^1_{ij}(\theta) = \left(\frac{p_i^\prime \theta_j}{\dot{\overline{q}}_j} \right)^2,  f^2_{ij}(\theta) = \left|\frac{p_i^{\prime\prime} \theta_j}{\ddot{\overline{q}}_j}\right| .
        \label{eqn_symmetric}
    \end{align}
Furthermore, the subgradient of $V(\theta)$ can be written as 
\begin{equation}
    \partial V(\theta) ={\bf{CO}}\left\{\nabla f_{ij}^k(\theta) \mid f_{ij}^k(\theta) = V(\theta)\right\},
    \label{eq:subgrad_cte}
\end{equation}
    where {\bf{CO}} denotes the convex hull, $i=1,\ldots,N$, $j=1,\ldots,n$ and $k=1,2$.
\end{coro}

When the limits are not antisymmetric, a similar result can be obtained, but four functions $f^k_{ij}$ will be necessary, one for each constraint in minimum and maximum velocity, minimum and maximum acceleration. 
From a practical perspective, this result implies that to take the descent direction, it is simply necessary to identify which constraint is currently maximizing the value of $V(\theta)$ and take the derivative at that joint $j$ and point in the path $i$. If more than one of these constraints is active simultaneously, picking one of them to calculate the derivative should suffice. Additionally, many of these terms would give the same result, as the function is not differentiable only if the active constraints are not linearly independent.

\begin{algorithm}[h]
\caption{INNERPROBLEM subroutine - Constant Speed}\label{alg:in_cte}
\begin{algorithmic}
\Require \\
$\theta \in \mathbb{R}^{nd}, p^\prime, p^{\prime\prime} \in \mathbb{R}^{1 \times d}$ \\
$\dot{\overline{q}},\dot{\underline{q}},\ddot{\overline{q}},\ddot{\underline{q}} \in \mathbb{R}^n$\\ 
\For{$i=0:N,\; j=0:n$}
\State $c_{\texttt{spd}_{ij}} = \max \left\{ \left({p'_{i}\theta_j}\right)/\dot{\overline{q}}_{j},\left({p'_{i}\theta_j}\right)/\dot{\underline{q}}_{j}\right\}$
\State $    c_{\texttt{acc}_{ij}} = \max\left\{\left({p^{\prime\prime}_{i}\theta_j}\right)/{\ddot{\overline{q}}_{j}},\left({p^{\prime\prime}_{i}\theta_j}\right)/{\ddot{\underline{q}}_{j}}\right\}$
\EndFor
\State $V(\theta) = \max \left\{c^2_{\texttt{spd}},c_{\texttt{acc}}\right\}$
\State Compute $\partial V(\theta)$ as in \eqref{eq:subgrad_cte}
\State \Return $V(\theta), \partial V(\theta)$
\end{algorithmic}
\end{algorithm}

The next section discusses analogous results in the case of a manipulator where the acceleration constraints are inactive. Before doing so, we present a pertinent remark regarding the choice of maximizing $t_f^2$ instead of $t_f$.

\begin{remark}
    \label{rmk:convex}
    At this point the reason for minimizing $t_f^2$ rather than $t_f$ can be made explicit. Although the solution to the problem is the same, this choice results in the value $V(\theta)$ being smoother and convex. Without the square, we would need to take square roots on the right-hand side of \eqref{eq_solution_inner_cte}. The introduction of the square root would result in a non-convex function, and a non-Lipschitz subgradient. Both convexity and Lipschitz continuity of gradients are desirable properties for optimization algorithms in general (see e.g., \cite{bertsekas1997nonlinear}).
\end{remark}

\subsection{No Acceleration Constraints Case}\label{sec_inner_noacc}

{{
Take in consideration the simplified problem described in \eqref{opt:inner_noacc}. This particular case can also have its value computed in closed form. This is the subject of the next Theorem.

{\color{black}{

\begin{theorem}\label{thm:inner_noacc}
Consider the optimization problem \eqref{opt:inner} parameterized by $\theta$. 
If the velocity and acceleration joint limits satisfy $\dot{\underline{q}}_{j} < 0$, $\dot{\overline{q}}_{j} > 0$, then $V(\theta)$, the value of the optimization problem \eqref{opt:inner_noacc}, is given by

\begin{dmath}\label{eq_solution_inner_noacc}
    V(\theta) =  \left(\sum_{i=0}^{N-1} \max_{j \in\{1,\dots,n\}}\left\{\frac{\Delta s_ip'_i\theta_j}{\dot{\overline{q}}_{j}},\frac{\Delta s_i p^{\prime}_i\theta_j}{\dot{\underline{q}}_{j}}\right\}\right)^2\!\!\!\!.
\end{dmath}

Furthermore, $V(\theta)$ is convex.

\end{theorem}
\begin{proof}
    The proof of \eqref{eq_solution_inner_noacc} is analogous to the proof of \eqref{eq_solution_inner_cte} in Appendix \ref{app:PF1}. In this case, there are only two potential solutions of $\dot{s}$ for each $i$ point in the path $s$ to consider, as there are no acceleration constraints. Compared to the general case, the lack of acceleration constraints implies that each decision variable $\dot{s}_i$ can be decided independently for each velocity constraint.
    
    The convexity of the value of the inner subproblem $V(\theta)$ follows from \eqref{eq_solution_inner_noacc} itself. First, the sum of convex functions is convex, so the sum of max functions is still convex. Next, the composition of two functions $f(g(x))$ is convex if both $f$ and $g$ are convex and f is increasing with respect to $g(x)$. As our sum of max functions is convex, non-negative, and monotonic increasing, then the convex square function is increasing with respect to it, and $V(\theta)$ is convex.
\end{proof}

\begin{coro}
    \label{cor:inner_noacc} 
    The subgradient of $V(\theta)$ in \eqref{eq_solution_inner_noacc} is well defined and can be computed in closed form.

    \begin{proof}
        
    Given that the closed form of  $V(\theta)$ is the sum of max functions, we can define its subdifferential as the sum of the convex hull of the subgradients for all active constraints for all the points in the path \cite{polyak1978}:
\begin{equation}
    \partial V(\theta) ={\bf{CO}}\left\{\nabla f_{ij}^k(\theta) \mid f_{ij}^k(\theta) = V(\theta)\right\},
    \label{eq:subgrad_noacc}
\end{equation}
    where {\bf{CO}} denotes the convex hull, $i=1,\ldots,N$, $j=1,\ldots,n$, $k=1,2$, and
        \begin{align}
        f^1_{ij}(\theta) = \frac{p_i^\prime \theta_j}{\dot{\overline{q}}_j},  f^2_{ij}(\theta) = \frac{p_i^\prime \theta_j}{\dot{\underline{q}}_j} .
        \label{eqn_symmetric}
    \end{align}

    \end{proof}
\end{coro}

From a practical perspective, this implies that to obtain the subgradient, it is necessary to identify which constraint is active at every point of the path, take the derivative with respect to $\theta$ at each point, and sum them. Note that given the independence of the parameters for each joint, it is possible to follow the derivative with respect to each joint at the same step when multiple constraints are active at the same point in the path.

The results in Theorem \ref{thm:inner_noacc} and Corollary \ref{cor:inner_noacc} provide insights to design a subroutine (see Algorithm \ref{alg:in_noacc}) that efficiently solves this particular case of the inner subproblem. 
The subroutine consists of evaluating all the terms involved in \eqref{eq_solution_inner_noacc} selecting the maximum for each point in the trajectory, then summing for all points.

\begin{algorithm}
\caption{INNERPROBLEM subroutine - No Acc Constraint}\label{alg:in_noacc}
\begin{algorithmic}
\Require \\
$\theta \in \mathbb{R}^{nd}, p^\prime, p^{\prime\prime} \in \mathbb{R}^{1 \times d}$ \\
$\dot{\overline{q}},\dot{\underline{q}},\ddot{\overline{q}},\ddot{\underline{q}} \in \mathbb{R}^n$\\ 
\For{$i=0:N,\; j=0:n$}
\State $c_{\texttt{spd}_{ij}} = \max \left\{ \left({p'_{i}\theta_j}\right)/\dot{\overline{q}}_{j},\left({p'_{i}\theta_j}\right)/\dot{\underline{q}}_{j}\right\}$
\EndFor
\State $V(\theta) = \left(\sum \left\{c_{\texttt{spd}}\right\}\right)^2$
\State Compute $\nabla V(\theta)$ as in \eqref{eq:subgrad_noacc}
\State \Return $V(\theta), \nabla V(\theta)$
\end{algorithmic}
\end{algorithm}

\section{Outer Subproblem Optimization}\label{sec_outer}

In this section we focus on solving Problem \eqref{opt:outer}. We assume that the inner problem is solved and that $V(\theta)$ and a descent direction $d$  can be evaluated as described in Section \ref{sec_inner_problem}. Note that the descent direction could take the form of the subgradient of $V(\theta)$} as in Section \ref{sec_inner_cte} and \ref{sec_inner_noacc}.

Problem \eqref{opt:outer} is a nonconvex optimization problem that could be solved with e.g., trust-region~\cite{bobrow1985}, active set~\cite{gill2019practical}, or projected gradient~\cite{combettes2011}  methods. In particular, we consider a primal-dual type of method (\cite[Ch. 11]{boyd2004convex}), to find a local minima, due to the simplicity of its implementation. Let $\lambda, \mu,\nu $ be in the positive orthants in $\mathbb{R},\;\mathbb{R}^{n(N+1)},\;\mbox{and}\;\mathbb{R}^{n(N+1)}$, respectively.  Define the Lagrangian for the outer subproblem \eqref{opt:outer} as
\begin{multline}   
    \label{eq:Lagrangian}
    \mathcal{L}_O(\theta,\lambda,\mu,\nu) = V(\theta) + \lambda \left(E(\theta) - \epsilon\right) \\ + \sum_{i=0}^N\sum_{j=1}^n \mu_{ij}(p_i\theta_j - \overline{q}_j)  -  \sum_{i=0}^N\sum_{j=1}^n \nu_{ij}(p_i\theta_j - \underline{q}_j).
\end{multline}
{The Cartesian path constraint $(E(\theta) - \epsilon)$ naturally depends on the highly nonlinear robot forward kinematics as seen in equation \eqref{eq:error}. If we differentiate $\chi_i = k(I_n\otimes p_i \theta)$ with respect to $\theta$, we obtain} 
\begin{equation}\label{eqn_error}
\nabla_\theta k(I_n\otimes p_i (\theta)) =   J(q_i)(I_n\otimes p_i).
\end{equation}
%
where $J(q_i)$ is the derivative of $k(q_i)$ with respect to $q_i$, commonly known as the manipulator Jacobian. 
Considering that the manipulator is redundant, the Jacobian $J(q_i)$ must have a nullspace~\cite[Ch. 10]{siciliano2007}, thus if $(I_n\otimes p_i)\Delta\theta$ is in the nullspace of $J(q_i)$ it is feasible to change the pose of the manipulator without changing the path followed by the end-effector. With this intuition, the last term in equation \eqref{eq:Lagrangian} is a correction term that exploits this redundancy to reduce the Cartesian error.  
Then define the following direction
\begin{multline}
    \label{eq:der_lag_theta}
    d_O = -d + \sum_{i=0}^N\sum_{j=1}^n \mu_{ij}p_i^T  -  \sum_{i=0}^N\sum_{j=1}^n \nu_{ij}p_i^T \\ +\lambda \left(\sum_{i=0}^N \frac{\partial E}{\partial \chi_i}J(q_i)(I_n\otimes p_i)\right),
\end{multline}
where $d$ is the descent direction defined in \eqref{eq:thm1} or the subgradient of $V(\theta)$. The remaining terms are the derivatives of the gradient of the Lagrangian with respect to $\theta$ (cf., \eqref{eq:Lagrangian}. This direction is the gradient of the Lagrangian when it is differentiable. 
Then, the primal update is equivalent to taking steps in the direction given by a weighted combination of a descent direction for $V(\theta)$ and the gradients of the constraints in \eqref{opt:outer}. The update rule can be written as
\begin{equation}
    \label{eq:update}
    \theta_{k+1} = \theta_{k} - \alpha d_O(\theta_k,\lambda_k,\mu_k,\nu_k),
\end{equation}
where $\alpha>0$ is the step-size. The linear combination weights are adaptively defined based on the constraint violation. Note that the constraint violations are the gradients of the Lagrangian with respect to the multipliers
\begin{equation}    
    \label{eq:der_lag_l1}
    \nabla_{\lambda} \mathcal{L}(\theta,\lambda,\mu,\nu) = \left(E(\theta) - \epsilon\right), 
\end{equation}   
\begin{equation}   
    \label{eq:der_lag_lij}
    \nabla_{\mu_{ij}} \mathcal{L}(\theta,\lambda,\mu,\nu) = (p_i\theta_j - \overline{q}_j), 
\end{equation}
\begin{equation}   
    \label{eq:der_lag_lji}
    \nabla_{\nu_{ij}} \mathcal{L}(\theta,\lambda,\mu,\nu) = (p_i\theta_j - \underline{q}_j).
\end{equation}
Then the updates of the Lagrangian multipliers $\lambda,\nu,\mu$ are given by
\begin{equation}
    \label{eq:update_lambda}
    \lambda_{k+1} = \max\{0,\lambda_{k} - \beta \nabla_\lambda \mathcal{L}(\theta_{k+1},\lambda_k,\mu_k,\nu_k)\},
\end{equation}
\begin{equation}
    \label{eq:update_mu}
    \mu_{k+1} = \max\{0,\mu_{k} - \beta \nabla_\mu \mathcal{L}(\theta_{k+1},\lambda_k,\mu_k,\nu_k)\},
\end{equation}
\begin{equation}
    \label{eq:update_nu}
    \nu_{k+1} = \max\{0,\nu_{k} - \beta \nabla_\nu \mathcal{L}\theta_{k+1},\lambda_k,\mu_k,\nu_k)\},
\end{equation}
where $\beta>0$ is the dual step-size. 
The algorithm discussed here is summarized in Algorithm \ref{alg:comp}. 
\begin{algorithm}
\caption{Complete Two-part Optimization Algorithm}\label{alg:comp}
\begin{algorithmic}
\Require \\
$\chi^d \in \mathbb{R}^{q\times N+1}, q \in \mathbb{R}^{n\times N+1}, s \in \mathbb{R}^{N+1}$ \\
$\dot{\overline{q}},\dot{\underline{q}},\ddot{\overline{q}},\ddot{\underline{q}} \in \mathbb{R}^n$, $K,\alpha_0,\epsilon,\sigma, \gamma,c_1 \in \mathbb{R}^+$

\State \Call{InnerProblem} {$p_i',p_i'',\theta,\dot{\underline{q}},\dot{\overline{q}},\ddot{\underline{q}},\ddot{\overline{q}}$}
\For{k = 1:K}
\Comment{Primal-Dual Descent}
\State Compute 
$d_o$ as in \eqref{eq:der_lag_theta}
\State Update $\theta_k$ as in \eqref{eq:update}
\State \Call{InnerProblem} {$p_i',p_i'',\theta,\dot{\underline{q}},\dot{\overline{q}},\ddot{\underline{q}},\ddot{\overline{q}}$}
\State Update $\lambda_k,\mu_k,\nu_k$ as in \eqref{eq:update_lambda}--
,\eqref{eq:update_nu}
\EndFor
\end{algorithmic}
\end{algorithm}

In the next section, we present numerical and experimental results to illustrate the performance of the algorithm.

\section{Numerical Simulations}\label{sec_numerical}

In this section, simulation results are presented to illustrate the proposed approach for optimizing joint trajectory for redundant manipulators under constant and non-constant path speed. The application is inspired by a cold spraying or material deposit application. In particular, we consider the planar problem depicted in Figure \ref{fig:robot}. The curve $\chi^d$ to be traced is the leading edge of a generic 2D fan blade mock-up for applications such as material deposit, or cold spraying. The trajectory has been discretized as $N +1 = 500$ points with uniform distance along the path, {and a specified Cartesian error tolerance of $\pm 1cm$ max, at any given point. In practice, it is not possible for the robot to immediately switch from a resting position to constant velocity or vice-versa, we address this by extending the desired curve by 10\% at the beginning and the end, and then disregarding the speed or cartesian errors produced on these extensions.}

In this planar case, the extra degree of freedom can be directly expressed by the heading variable $\phi$. The manipulator is such that the length of its links are $a_1 = 2m$, $a_2 = 1.5m$, $a_3 = 1m$. We impose the following joint speed and acceleration limits $\dot{\overline{q}} = -\dot{\underline{q}} =[1.75\;1.57\;1]\;\;rad/s$ and $\ddot{\overline{q}} = -\ddot{\underline{q}} = [35\;31.4\;20]\;\;rad/s^2$. These values are based on industrial robot manipulators, such as a FANUC M-1000iA series robot, to handle the reach and size of the fanblade.\footnote{The typical industrial robot of this size would usually have at least 6 degrees of freedom, we have reduced it to a 3R manipulator for the purposes of this numerical demonstration.}.

\begin{figure}
    \centering
    \includegraphics[width=4cm]{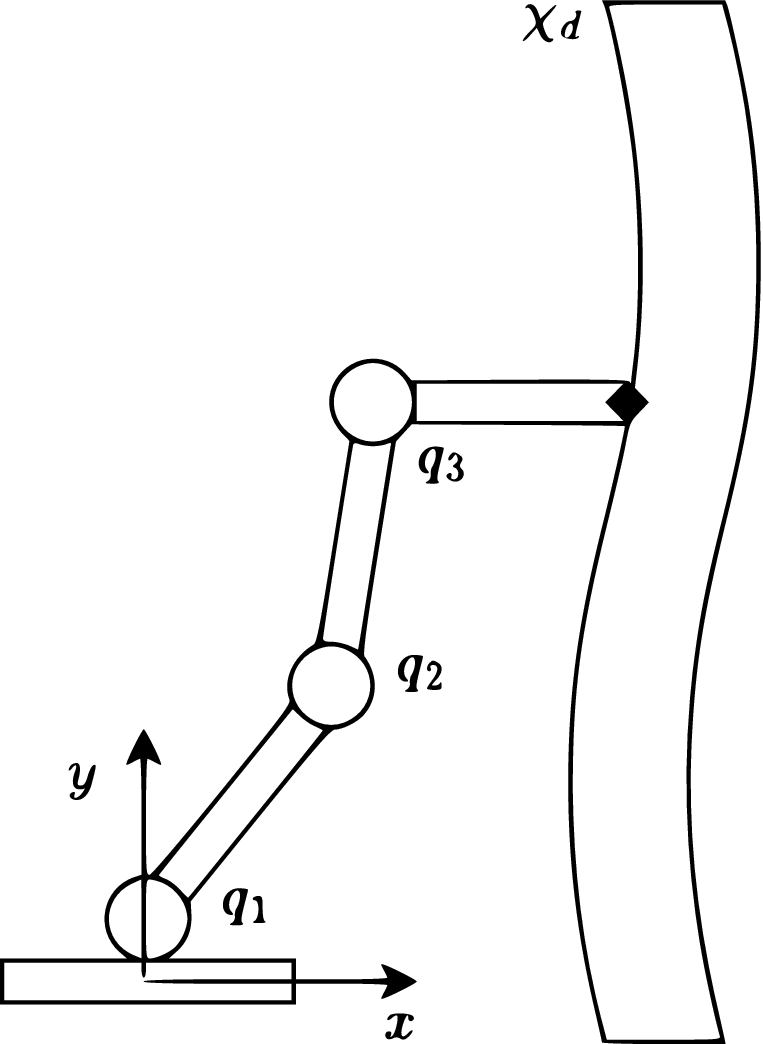}
    \caption{Reference task: A 3R Planar Robot manipulator to trace the 2D leading edge of a fanblade for a material deposit application. The free variable for this particular simulated scenario is the rotation along the curve edge.}
    \label{fig:robot}
\end{figure}

    The forward kinematics mapping $k(q)$ for a 3R planar robot manipulator is given as follows
    \begin{equation}
        \left[\begin{array}{c}
             x  \\
             y  \\
             \phi
        \end{array}\right] = \left[\begin{array}{c}
             a_1c_1 + a_2c_{12}+a_3c_{123}  \\
             a_1s_{1} + a_2s_{12}+a_3s_{123}  \\
             q_1 + q_2 + q_3
        \end{array}\right],
    \end{equation}
    where $c_{ijk} = \cos(q_i + q_j + q_k)$ and $s_{ijk} = \sin(q_i + q_j + q_k)$.

    Note that as $\phi$ is our free variable, the important constraints for the optimization problem are the first two rows of $k(q)$ and $J(q)$.

    Furthermore, for the next two subsections, the initial condition for the free variable $\phi$ has been randomly picked between $[0, 2\pi]$ with uniform distribution and kept constant for the whole trajectory, for a total of ten experiments. The result of this draw is $\phi$ $\in$ $\left[5.12 \;    0.225 \;   5.34 \;    5.87 \;    4.26 \;    4.76 \;    4.67 \;    2.46\;    2.12 \;    1.08\right]$. For each $\phi$, the corresponding initial $q$ has been calculated by inverse kinematics $k^{-1}$ and parameterized, in terms of $s$, using a polynomial basis function with degree equal to five.

Simulations are carried out in a Precision 3650 Tower Dell Workstation, with
 11th Gen Intel(R) Core(TM) i9-11900K, 128 GB RAM, running Windows 11 Pro and Matlab R2021b. 

In Section \ref{sec_numerical_ours} we present the results of our proposed approach for different initial conditions of the free variable $\phi$, with the same hyperparameters. Afterward (Section \ref{sec_comparison}), we offer a comparison of the performance of our proposed approach to Matlab's innate nonlinear optimization solver, for the same initial parameters, using the Interior Point methods for nonlinear optimization (see e.g.,~\cite{byrd2000trust,waltz2006interior}) and Active-set algorithms (see e.g.,~\cite{gill2019practical}).

    \subsection{Results for Proposed Approach}\label{sec_numerical_ours}
\begin{figure}
        	\centering
            \psfrag{x}{\hspace{-3mm}\scriptsize $x (m)$}
        	\psfrag{y}{\hspace{-3mm}\scriptsize $y (m)$}
        	\psfrag{chid}{\scriptsize $\chi^d$}
        	\psfrag{chi1}{\scriptsize $\chi_1$}
        	\psfrag{chi2}{\scriptsize $\chi_2$}
        	\psfrag{chi3}{\scriptsize $\chi_3$}
        	\psfrag{chi4}{\scriptsize $\chi_4$}
        	\psfrag{chi5}{\scriptsize $\chi_5$}
        	\psfrag{chi6}{\scriptsize $\chi_6$}
        	\psfrag{chi7}{\scriptsize $\chi_7$}
        	\psfrag{chi8}{\scriptsize $\chi_8$}
        	\psfrag{chi9}{\scriptsize $\chi_9$}
        	\psfrag{chi10}{\scriptsize $\chi_{10}$}
        	\includegraphics[width=8cm]{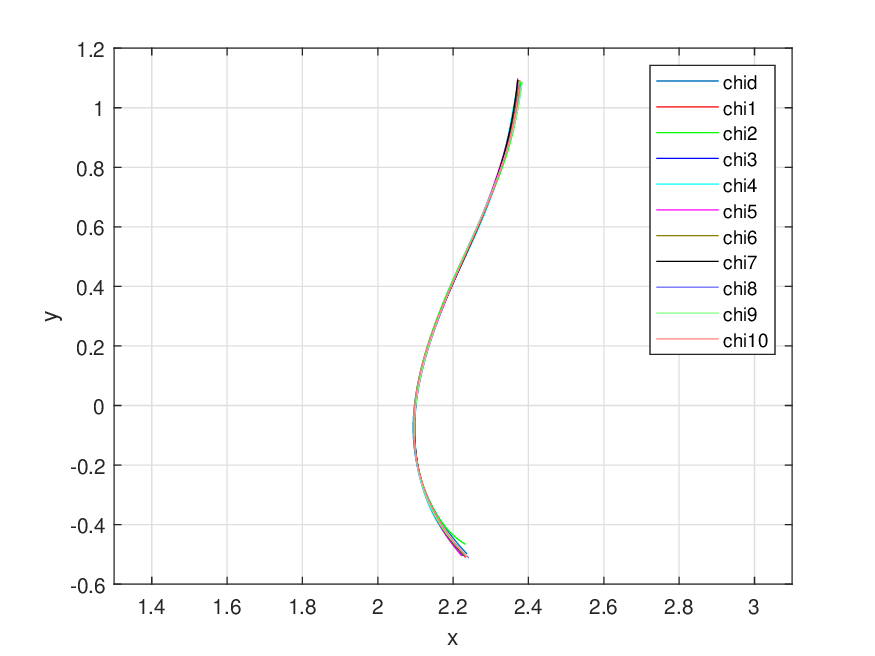} 
            \hfill 
            \psfrag{s}{\hspace{-3mm}\scriptsize $s$}
        	\psfrag{e}{\hspace{-3mm}\scriptsize $e (m)$}
        	\psfrag{e1}{\scriptsize $e_1$}
        	\psfrag{e2}{\scriptsize $e_2$}
        	\psfrag{e3}{\scriptsize $e_3$}
        	\psfrag{e4}{\scriptsize $e_4$}
        	\psfrag{e5}{\scriptsize $e_5$}
        	\psfrag{e6}{\scriptsize $e_6$}
        	\psfrag{e7}{\scriptsize $e_7$}
        	\psfrag{e8}{\scriptsize $e_8$}
        	\psfrag{e9}{\scriptsize $e_9$}
        	\psfrag{e10}{\scriptsize $e_{10}$}
            \includegraphics[width=8cm]{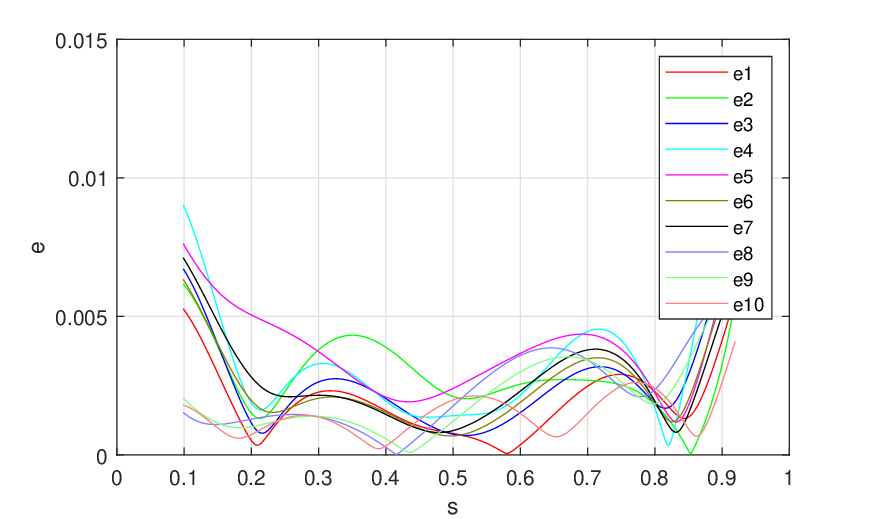}
	      \caption{Final Cartesian path $\chi$ and error with the reference path for the trajectory obtained with Algorithm \ref{alg:comp} for ten different initial conditions. For all experiments, the error has stayed under the designed threshold. }
	      \label{fig:sim1}
            \end{figure}
            \begin{figure}
        	\centering
            \psfrag{it}{\hspace{-3mm}\scriptsize Iterations}
        	\psfrag{tf}{\hspace{-3mm}\scriptsize $t_f (s)$}
        	\psfrag{tf1}{\hspace{-1mm}\scriptsize $t_{f_{1}}$}
        	\psfrag{tf2}{\hspace{-1mm}\scriptsize $t_{f_{2}}$}
        	\psfrag{tf3}{\hspace{-1mm}\scriptsize $t_{f_{3}}$}
        	\psfrag{tf4}{\hspace{-1mm}\scriptsize $t_{f_{4}}$}
        	\psfrag{tf5}{\hspace{-1mm}\scriptsize $t_{f_{5}}$}
        	\psfrag{tf6}{\hspace{-1mm}\scriptsize $t_{f_{6}}$}
        	\psfrag{tf7}{\hspace{-1mm}\scriptsize $t_{f_{7}}$}
        	\psfrag{tf8}{\hspace{-1mm}\scriptsize $t_{f_{8}}$}
        	\psfrag{tf9}{\hspace{-1mm}\scriptsize $t_{f_{9}}$}
        	\psfrag{tf10}{\hspace{-1mm}\scriptsize $t_{f_{10}}$}

        	\includegraphics[width=8cm]{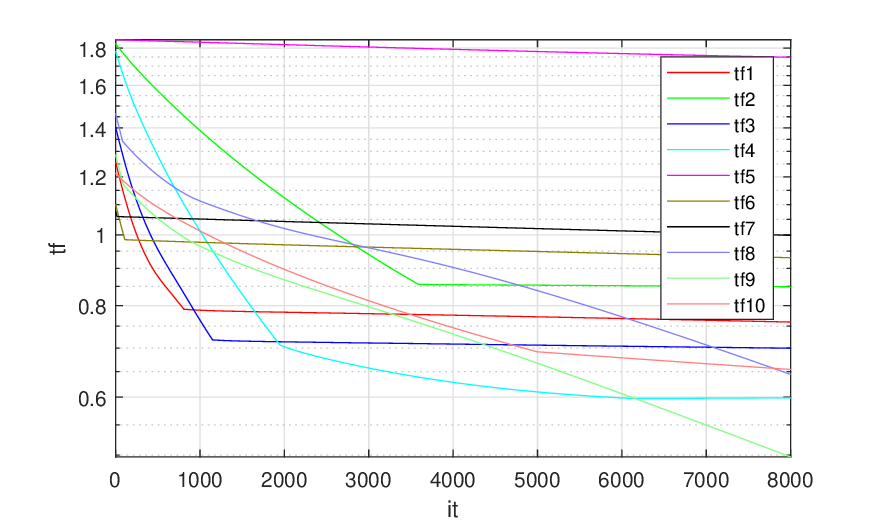} 
            \hfill 
        	\psfrag{it}{\hspace{-3mm}\scriptsize Iterations}
        	\psfrag{cv}{\hspace{-12mm}\scriptsize Constraint Tolerance Usage}
        	\psfrag{cv1}{\hspace{-1mm}\scriptsize $ct_1$}
        	\psfrag{cv2}{\hspace{-1mm}\scriptsize $ct_2$}
        	\psfrag{cv3}{\hspace{-1mm}\scriptsize $ct_3$}
        	\psfrag{cv4}{\hspace{-1mm}\scriptsize $ct_4$}
        	\psfrag{cv5}{\hspace{-1mm}\scriptsize $ct_5$}
        	\psfrag{cv6}{\hspace{-1mm}\scriptsize $ct_6$}
        	\psfrag{cv7}{\hspace{-1mm}\scriptsize $ct_7$}
        	\psfrag{cv8}{\hspace{-1mm}\scriptsize $ct_8$}
        	\psfrag{cv9}{\hspace{-1mm}\scriptsize $ct_9$}
        	\psfrag{cv10}{\hspace{-1mm}\scriptsize $ct_{10}$}
            \includegraphics[width=8cm]{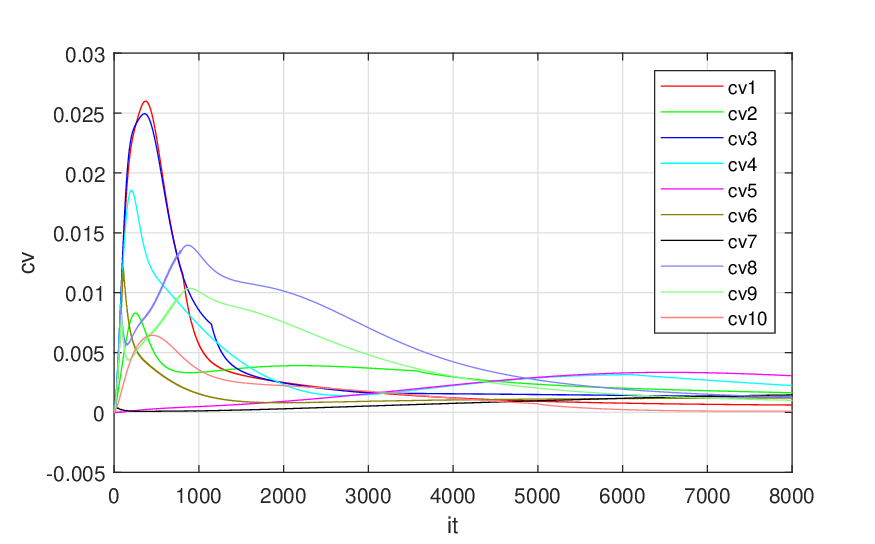}
	      \caption{Evolution of the objective function and constraint tolerance usage of Algorithm \ref{alg:comp} with the number of iterations for ten different initial conditions. Though improvement exists for all experiments, one can see that performance is doubtlessly dependant on initial conditions.}
	      \label{fig:sim2}
            \end{figure}
        In this subsection, we show results for the performance of our proposed approach for the ten different initial values of $\phi$, while hyperparameters have been unchanged between runs. We set the hyperparameters as 8000 iterations for the Complete Two-part Optimization Algorithm;               Path constraint tolerance $\epsilon = 0.00001$ ;  Step size $\alpha = 0.00001$, step size $\beta = 0.5$.

            Optimization results are presented from Figures \ref{fig:sim1} to \ref{fig:sim4}. The Cartesian trajectory achieved by the proposed optimization for all different initial conditions can be seen in Figure \ref{fig:sim1} (top), where we can notice that the algorithm has pushed most of the admissible error towards the end of the trajectory. In Figure \ref{fig:sim2} (bottom), we can see that most runs have violated the given tolerance $\epsilon$ for the norm 2 path constraint, however, Figure \ref{fig:sim1} (bottom) shows that none of the runs have violated task specifications of $1 \mbox{cm}$ max Cartesian error. Figure \ref{fig:sim1} (bottom) also reinforces the notion that most of the error has been pushed towards the end of the curve.

            We can see how the travel time $t_f$ improves (c.f., Figure \ref{fig:sim2} (top)), over every iteration of the algorithm for all $\phi$, and, on average, we have an improvement of nearly $38\%$ when comparing the travel time of the initial condition and that of the final trajectory obtained. More detailed results and averages can be seen in Table \ref{tab:proposed}, where we have information on the running time of the algorithm, on top of results and averages for variables $\dot{s}$, $t_f$, max Cartesian error, and the $t_f$ improvement percentage.

            In Fig. \ref{fig:sim4}, we see how the proposed optimization algorithm has changed the initial choice for the free variable $\phi$, for each initial condition, and how relatively small changes have significantly improved the performance of the manipulator.

            {Furthermore, Figure \ref{fig:sim5} illustrate how different choices of $\phi$ can result in entirely different joint curves with different path speeds, for a same Cartesian path. In particular, we show the joint curve for $\phi_5$ and $\phi_9$, the slowest and fastest joint paths obtained by the optimization algorithm, respectively. It is noticeable that the faster trajectory ($\phi_9$) has much smaller joint variation from start to end of path than the slower trajectory. The intuition for this is that since a joint is limited by its joint speed and acceleration, {smaller derivatives with respect to the path $s$ allow for higher path speed limits.} This idea is corroborated by the result of Theorem \ref{thm:inner} in Section \ref{sec_inner_problem}, smaller values of $p''_i\theta_j$ and $p'_i\theta_j$ allow for higher values of $\dot{s}$ before activating their respective constraint. 

            \begin{table*}[!ht]
            \centering
\begin{tabular}{|c|c|c|c|c|c|c|c|c|c|c|c|}
\hline
Proposed Approach         & $\phi_1$ & $\phi_2$ & $\phi_3$ & $\phi_4$ & $\phi_5$ & $\phi_6$ & $\phi_7$ & $\phi_8$ & $\phi_9$ & $\phi_{10}$ & \textbf{Mean}                \\ \hline \hline
$\dot{s}$                 & 1.3160 &	1.1777 &	1.4288 &	1.6700 &	0.5747 &	1.0753 &	1.0020 &	1.5516 &	2.0117 &	1.5277      & \textbf{1.3335 $\pm$	0.4006} \\ \hline
$t_f$ (s)                    & 0.7599&	0.8491	&0.6999&	0.5988	&1.7400	&0.9300	&0.9980	&0.6445&	0.4971 &	0.6546      & \textbf{0.8372$\pm$	0.3524} \\ \hline
Running Time (s)             & 207 &	210&	209&	209&	208&	205&	205&	205&	205&	205      & \textbf{207 $\pm$	2} \\ \hline
Max Path Error $(mm)$ & 5.7 &	6.2	& 8.0 &	10.0 &	8.0 &	7.0 &	7.0 &	6.0 &	6.5 &	4.0            & \textbf{6.8	$\pm$1.6} \\ \hline
$t_f$ Improvement               & 0.2126   & 0.5271   & 0.5217   & 0.5548   & 0.0978   & 0.3776   & 0.2881   & 0.3883   & 0.4786   & 0.4077      & \textbf{0.3854 $\pm$ 0.1486} \\ \hline
\end{tabular}
\caption{Individual and average results of the simulation for ten different initial conditions. As expected of a local optimization approach, even as running time is around the same for all experiments, performance improvement and average error have higher standard deviations. Regardless, even in the worst run the algorithm has shown a 9.7\% improvement in final time, with improvement going as high as as 55.5\%. }\label{tab:proposed}
\end{table*}
            \begin{figure}
        	\centering
        	\psfrag{s}{\hspace{-3mm}\scriptsize $s$}
        	\psfrag{phi}{\hspace{-3mm}\scriptsize $\phi\;(rad)$}
        	\psfrag{phi1}{\hspace{0mm}\scriptsize $\phi_1$}
        	\psfrag{phi2}{\hspace{0mm}\scriptsize $\phi_2$}
        	\psfrag{phi3}{\hspace{0mm}\scriptsize $\phi_3$}
        	\psfrag{phi4}{\hspace{0mm}\scriptsize $\phi_4$}
        	\psfrag{phi5}{\hspace{0mm}\scriptsize $\phi_5$}
        	\psfrag{phi6}{\hspace{0mm}\scriptsize $\phi_6$}
        	\psfrag{phi7}{\hspace{0mm}\scriptsize $\phi_7$}
        	\psfrag{phi8}{\hspace{0mm}\scriptsize $\phi_8$}
        	\psfrag{phi9}{\hspace{0mm}\scriptsize $\phi_9$}
        	\psfrag{phi10}{\hspace{0mm}\scriptsize $\phi_{10}$}
        	\includegraphics[width=8cm]{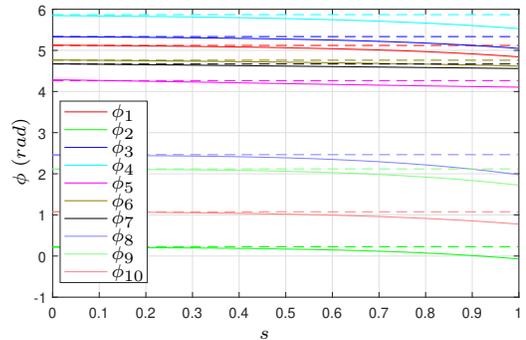} 
	      \caption{Initial and Final configuration of the free variable $\phi$ over the path length for all initial conditions. Even relatively small changes in the free variable can drastically improve performance.}
	      \label{fig:sim4}
            \end{figure}  
            
        \begin{figure}
    \centering
    \includegraphics[width=8cm]{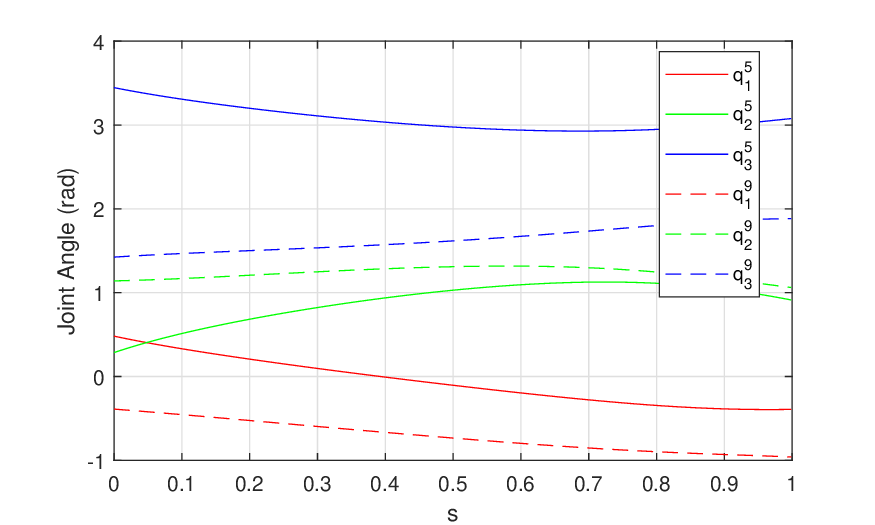}
    \caption{Comparison between the optimal joint trajectories obtained for $\phi_5$ (solid lines) and $\phi_9$ (dashed lines). Though they trace the same Cartesian curve, experiment 9 ended with a final time less than half of experiment 5. Note that the joint paths in experiment 9 have gentler slopes and less curvature.}
    \label{fig:sim5}
\end{figure}

    \subsection{Performance Comparison with other nonlinear optimization algorithms}\label{sec_comparison}

    In this subsection, we compare the performance of our proposed approach to two of Matlab innate nonlinear optimization solvers, using interior point methods (see e.g.,~\cite{byrd2000trust,waltz2006interior}) and Active-set algorithms (see e.g.,~\cite{gill2019practical}). We ran the two baseline algorithms for the same initial conditions as the proposed algorithm (Section \ref{sec_numerical_ours}). For the Matlab solvers, we have instructed both Algorithms to solve problem \eqref{opt:generaldiscrete}. Initial velocity guesses for these methods were given by the solution of our inner problem.

            \begin{figure}
        	\centering
        	\psfrag{art}{\hspace{-12mm}\scriptsize Algorithm Running Time $(s)$}
        	\psfrag{tf}{\hspace{-3mm}\scriptsize $t_f\;(s)$}
        	\includegraphics[width=8cm]{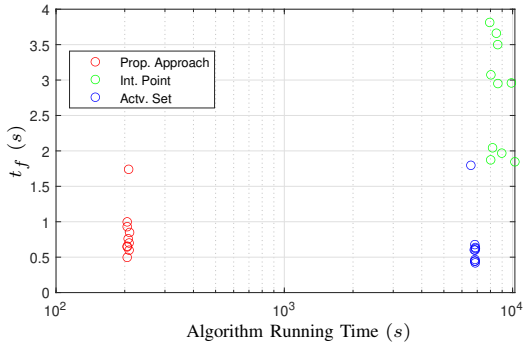}
	      \caption{Performance versus algorithm running time for the three compared methods. The active-set method has the best performance but is far more computationally expensive. Our proposed approach finds a solution with comparable performance over an order of magnitude faster. The interior point method fails to find a good solution with the resources given.}
	      \label{fig:comp}
            \end{figure}

            \begin{table*}[!ht]
            \centering
\begin{tabular}{|c|c|c|c|c|c|}
\hline
                  & $\dot{s}$           & $t_f(s)$               & Running Time (s)       & Max Path Error $(mm)$& $t_f$ Improvement         \\ \hline \hline
Proposed Approach & 1.3335 $\pm$ 0.4006 & 0.8372 $\pm$ 0.3524 & 207 $\pm$ 2 & 6.8 $\pm$ 1.6       & 0.4020 $\pm$ 0.2248 \\ \hline
Interior Point    & 0.3896 $\pm$	01144       & 2.7689	$\pm$ 0.7749       & 8694 $\pm$	792  & 3	$\pm$ 1            & -0.9874	$\pm$ 0.5732       \\ \hline
Active-set        & 1.7566$\pm$	0.5403       & 0.6710 $\pm$	0.4066      & 6815$\pm$	91   & 5.7	$\pm$1.9           & 0.5141 $\pm$	0.2856     \\ \hline
\end{tabular}
\caption{Mean and standard deviation of relevant metrics for 
the optimization approaches, in simulation. Though the active-set approach has obtained higher speeds than our method on average, our proposed approach reaches a result faster by a factor of thirty, given the same initial conditions.}\label{tab:comp}
\end{table*}

The hyperparameters for interior point and active set methods were defined as follows: A maximum of 100 iterations for the Matlab solvers; Path violation tolerance $\epsilon = 0.00001$. Other parameters were kept at their default values.

Average results over the proposed initial conditions are shown in Table \ref{tab:comp} and Figure \ref{fig:comp}. 
Our algorithm has mostly converged to a local solution with an average of 207 seconds, while the Interior Point algorithm average run takes 8690 seconds, about 2.4 hours, and the Active-set algorithm has an average of 6815 seconds, about 1.9 hours. 
\begin{figure}
        	\centering
            \psfrag{Nominal}{\hspace{-5mm} \scriptsize Nom. Optimized}
        	\includegraphics[width=8cm]{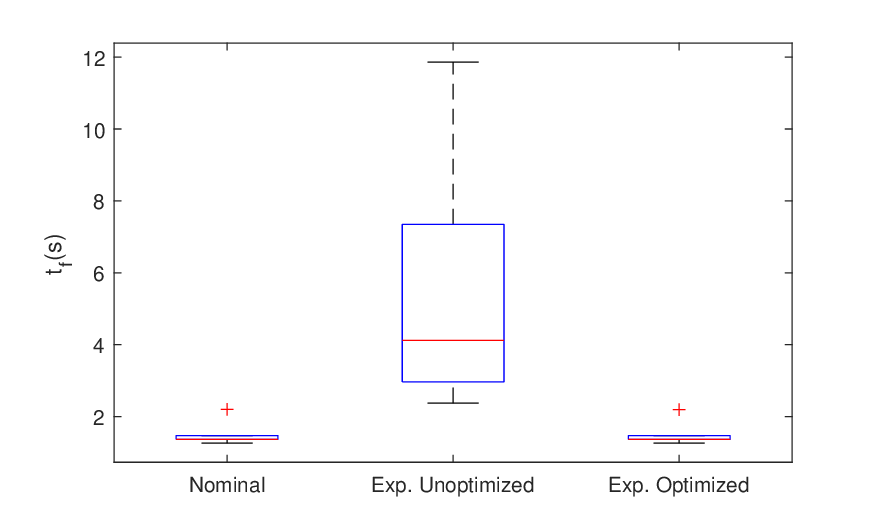} 
            \hfill 
        	\includegraphics[width=8cm]{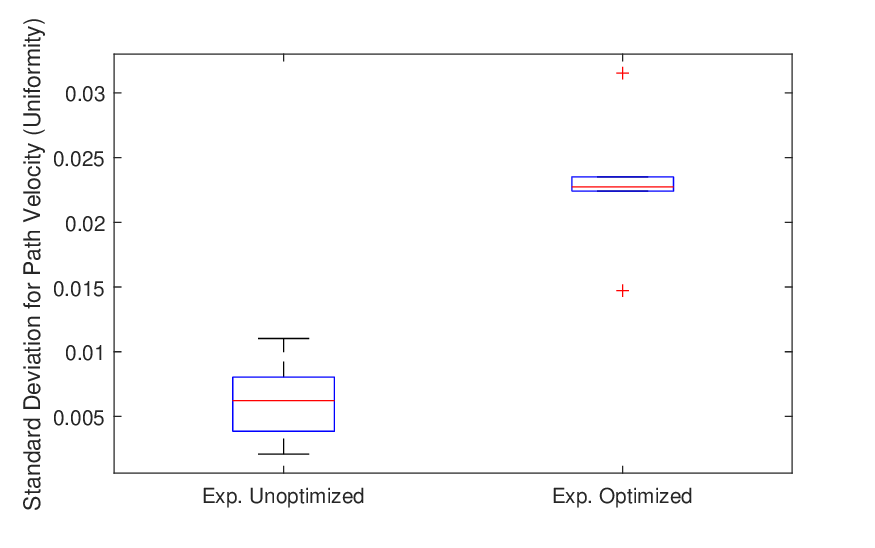} 
	      \caption{Trajectory Time and Motion Uniformity for Optimized Nominal, Unoptimized and Optimized Experimental Results. Experimental results were as fast as expected in terms of trajectory time and improvement compared to unoptimized results is noticeable. The faster path has a higher speed standard deviation when compared to the slower path, which is due to acceleration and deceleration at the edges of the curve.}
	      \label{fig:exp1}
            \end{figure}
            \begin{figure}
        	\centering
            \psfrag{Nominal}{\hspace{-5mm} \scriptsize Nom. Optimized}
        	\includegraphics[width=8cm]{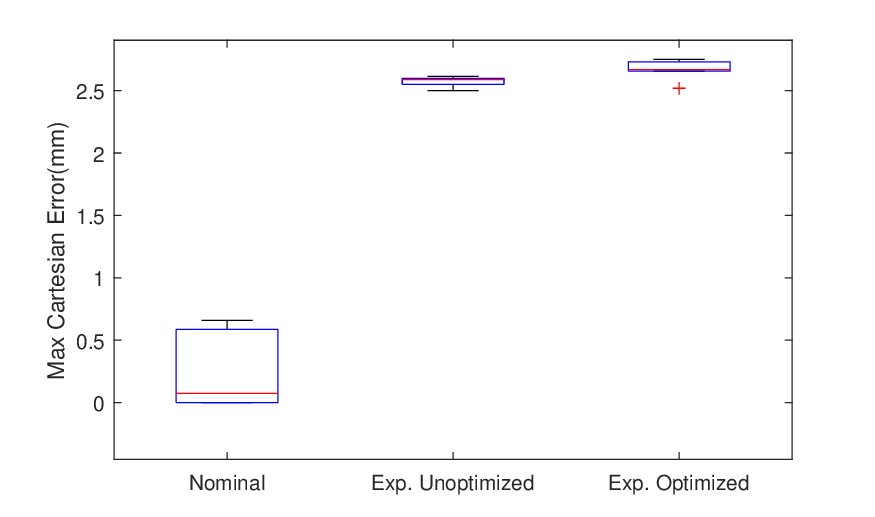} 
            \hfill 
            \includegraphics[width=8cm]{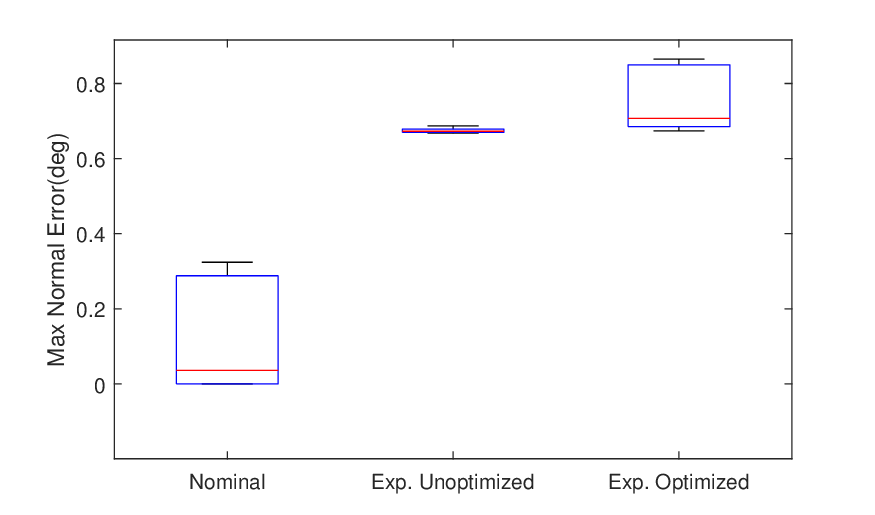}
	      \caption{Max Cartesian and Normal Errors for Optimized Nominal, Unoptimized and Optimized Experimental Results. Both errors were higher than what had been expected from the nominal case, yet acceptable when taking in consideration the experimental error for the unoptimized curve. Even then, errors are still under the desired threshold.}
	      \label{fig:exp2}
            \end{figure}
            \begin{figure}
        	\centering
        	\includegraphics[width=8cm]{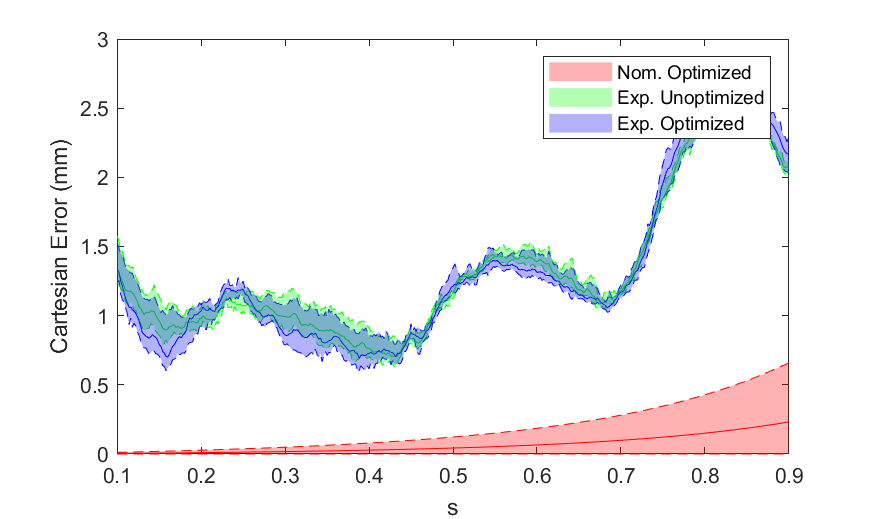} 
            \hfill 
            \includegraphics[width=8cm]{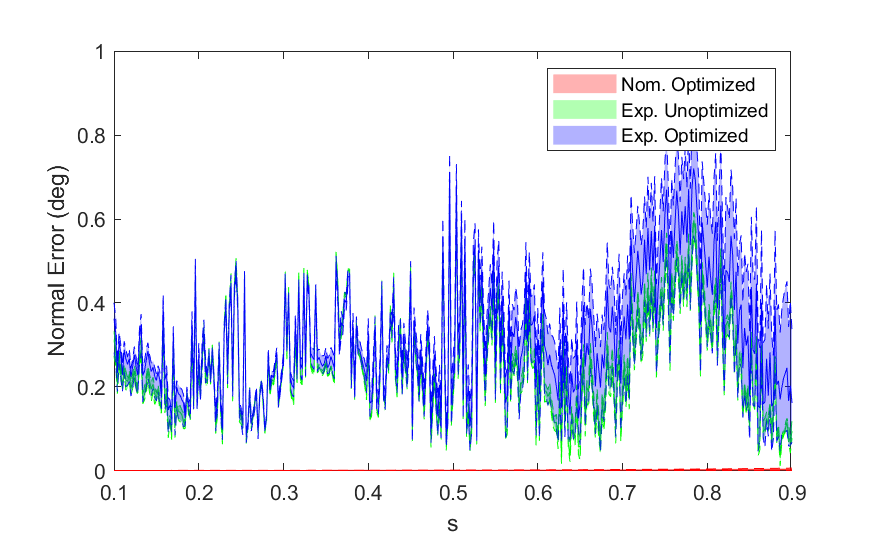}
	      \caption{Cartesian and Normal Errors for Optimized Nominal, Unoptimized and Optimized Experimental Results. Shows comparisons not only for the maximum error but also for the error along the trajectory. Note that it does have the same tendency of pushing errors towards the edge of the curve.}
	      \label{fig:exp3}
            \end{figure}
Another point of importance is the reliability of those algorithms. We can see that, on average, our proposed approach has better improvement than the interior point method, and comparable performance with the active-set methods, about 40\% against -98\% and 50\%, respectively. We can see in Figure \ref{fig:comp} that the active-set algorithm had solutions that managed to beat our proposed approach, with similar performances in terms of Cartesian path error, given about thirty times the computational time to reach its solution. 

The interior point algorithm was reliable in terms of keeping the error within bounds, never exceeding the constraint feasibility, but the price paid was having each run exceed the limit number of iterations, with over two hours of run time, without producing any meaning result. Meanwhile, the other approaches have produced results with faster speeds while keeping the Cartesian error under the $1cm$ specification.

In conclusion, our proposed approach has a number of advantages over the Interior Point and Active-set methods when it comes to the proposed problem: it is orders of magnitude faster to run, has a good degree of reliability, and produces comparable, if not better, results given the resources expended. The two other algorithms are powerful in the sense that they can tackle a variety of general nonlinear problems but are unable to exploit the particularities and structure of the path traversal time optimization problem that we solve.

In the next section, we present the results obtained with a 6R robot, for a similar edge-tracing task.

\section{Experimental Results}
\label{sec_exp}

In this section, experimental results are presented to illustrate an actual application for the proposed joint trajectory optimization approach. The application is inspired by a cold spraying or material deposit application as in the last section. For this experiment, we consider a generic 3D fan blade, and a 6R robot. In this case, the free variable is rotation around the normal of the Cartesian trajectory to be traced. We discretize the trajectory as N+1 = 500 points with uniform distance along the path. We have a specified Cartesian error tolerance of $\pm5mm$ and an angular error around the normal of $1$ degree max, at any given point. Once again, since in practice the robot cannot switch from a resting position to constant velocity, we extend the desired curve by $10\%$ at the beginning and the end, disregarding speed, position, or orientation errors produced on these extensions.

In this 3D case, the extra degree of freedom can be expressed by the rotation around the trajectory normal. The 6R manipulator used for the following experimental results is the UR10e, from the manufacturer Universal Robotics. Denavit-Hartenberg (DH) parameters for this manipulator can be found directly on the company website at \cite{ur10datasheet}. The expression for the forward kinematics $k(q)$ and the differential kinematics $J(q)$ can be directly determined from the DH parameters, as seen in \cite{ur10DHparameters}.

Furthermore, for the next two subsections, the free variable has been picked in two different ways: for one experiment, the free variable is described by a constantly increasing curve, from $0$ to $\pi$; for the next set of five experiments, ten coefficients for a polynomial curve have been picked with a $\mathcal{N}(0,2)$ distribution. The corresponding initial $q$ curves have been calculated by inverse kinematics $k^{-1}$.

The optimization process was carried out in a Precision 3650 Tower Dell Workstation, with
 11th Gen Intel(R) Core(TM) i9-11900K, 128 GB RAM, running Windows 11 Pro and Matlab R2021b. The optimized curves were then executed in the UR10e manipulator, through a Lenovo Station Model T-15g running Ubuntu 18.04, ROS Melodic, and the UR Robot Driver package. Data of each run has been saved in a rosbag for analysis.

 In Section \ref{sec_real_c}, we present the results of our proposed optimization approach for constant velocity, for different initial conditions of the free variable. In Section \ref{sec_real_nc}, we present the results of the optimization approach for non-constant velocity, for the same set of initial conditions of the free variable.

\subsection{Real Robot Constant Velocity}\label{sec_real_c}

            \begin{figure}
        	\centering
\includegraphics[width=8cm]{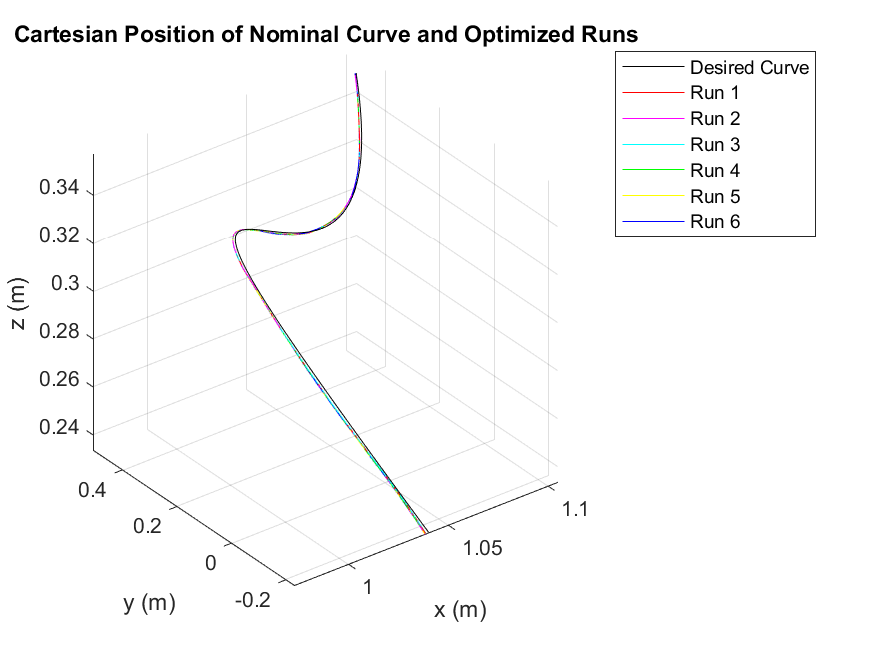} 
	      \caption{Cartesian Path for all Experiments compared to Nominal Curve. This figure visually shows how close all optimized experiments are to the desired curve, despite the increased error.}
	      \label{fig:exp5}
            \end{figure}
        In this subsection, we show results for the performance of our proposed approach for constant velocity case, presented in Algorithm \ref{alg:in_cte} and \ref{alg:comp}. For the six different initial values of the free variable, the hyperparameters for all runs are set as follows Step-size $10^{-4}$. Fixed $\lambda = 10$ for the Cartesian Error, $\lambda = 0$ for the orientation. Each experiment was terminated when the value reached no improvement under given error tolerance. Experimental results for the constant velocity case are presented from Figures \ref{fig:exp1} to \ref{fig:exp5}. 

            Figure \ref{fig:exp1}(a) shows that the final curve time from the experimental results matches what had been expected from simulations. The same figure also illustrates that despite the large discrepancies in time for the various initial curves, all curves converge to very similar final times. Both simulated and experimental results have an average improvement of 63.5\%.
            
            Figure \ref{fig:exp1}(b) showcases the motion uniformity the constant velocity has been reached with a good degree of success, with low standard deviation for the average path speed. This deviation comes from an acceleration and deceleration phenomenon at the edges of the curve and the internal interpolation of the driver to reach the required speed.

            We have executed both the initial and optimized curves for each experiment, and we can see that the time improvement between them matches what had been expected from simulation. The major difference between simulation and experimental results has been the cartesian and normal errors. We can see, however, that they consistent with the error already obtained from running the unoptimized curve. Figure \ref{fig:exp2}(a) and Figure \ref{fig:exp2}(b) illustrate this effect. The maximum error however, is still under our proposed limit of 5 millimeters in Cartesian and 1 degree around the normal.

            Figure \ref{fig:exp3} and Figure \ref{fig:exp5} represent the errors along the path for nominal optimized, experimental unoptimized, and experimental optimized curves. The curve spread in Figures \ref{fig:exp3}(a) and (b) represent the minimum, average, and maximum errors for each point in the path. Figure \ref{fig:exp5} show that the Cartesian curves obtained in each experiment are very close to the nominal curve.
            
            With this, we conclude that the curve obtained by the optimization algorithm is reasonable and can be implemented in the actual robot with satisfactory results. In the next subsection, we analyze the experimental results for the case with variable velocity.

\subsection{Real Robot non-constant Velocity}\label{sec_real_nc}

\begin{figure}
        	\centering
        	\psfrag{Nominal}{\hspace{-5mm} \scriptsize Nom. Optimized}
            
    \includegraphics[width=8cm]{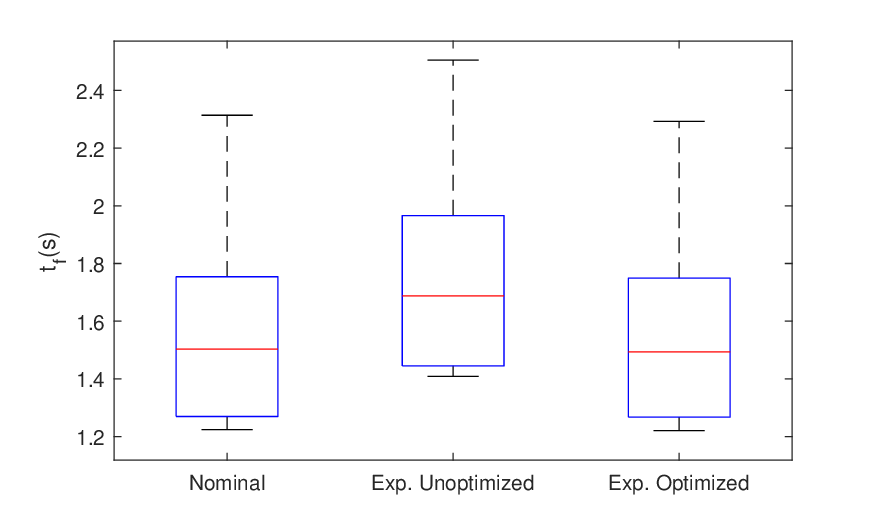} 
            \hfill 
        	\includegraphics[width=8cm]{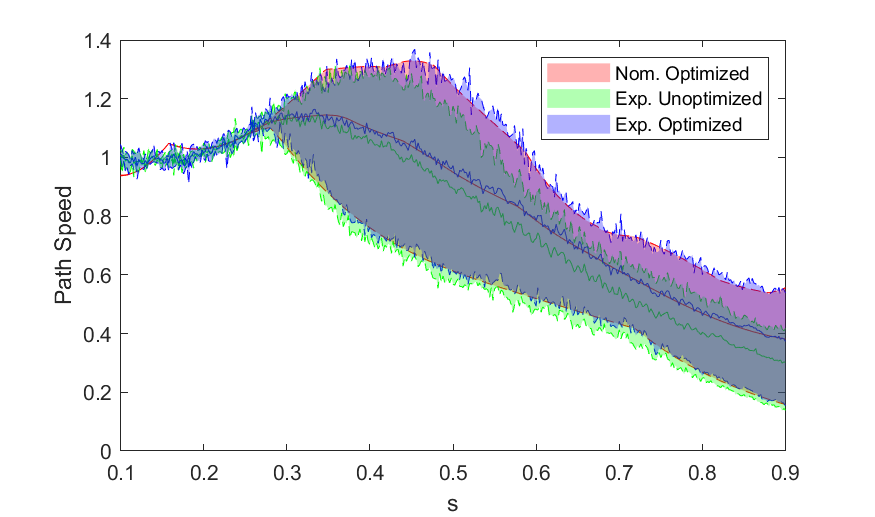} 
	      \caption{Trajectory Time and Path Speed for Optimized Nominal, Unoptimized and Optimized Experimental Results. Experimental results were as fast as expected in terms of trajectory time and improvement is more subdued compared to the constant velocity case, but still noticeable at about 11.5\%. Note that by the standard deviation, there is not much improvement to be done at the beginning of the path, although better results could be achieved close to the end.}
	      \label{fig:exp6}
            \end{figure}
            \begin{figure}
        	\centering
        	\includegraphics[width=8cm]{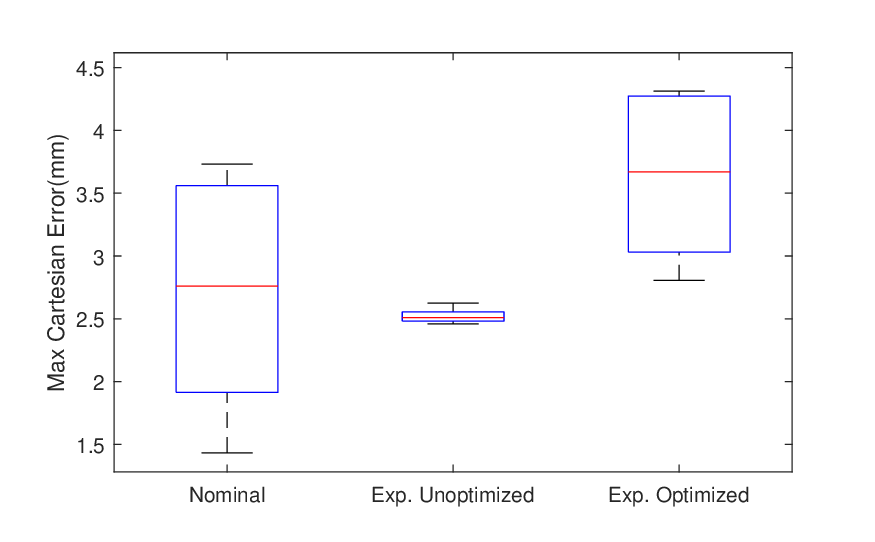} 
            \hfill 
            \includegraphics[width=8cm]{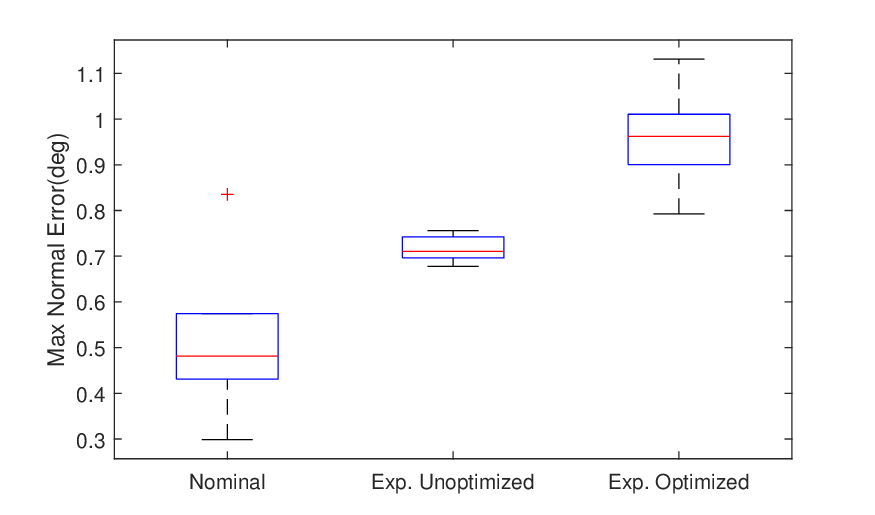}
	      \caption{Max Cartesian and Normal Errors for Optimized Nominal, Unoptimized and Optimized Experimental Results. Cartesian errors have been kept under 5mm even with the average increase on the experimental runs, but angular error has gone slightly above the limit of 1 degree in one of the experimental runs. One possible solution would be tightening up the constraint for the normal error and rerunning the optimization setup.}
	      \label{fig:exp7}
            \end{figure}
            \begin{figure}
        	\centering
            \psfrag{Nominal}{\hspace{-5mm} \scriptsize Nom. Optimized}         
        	\includegraphics[width=8cm]{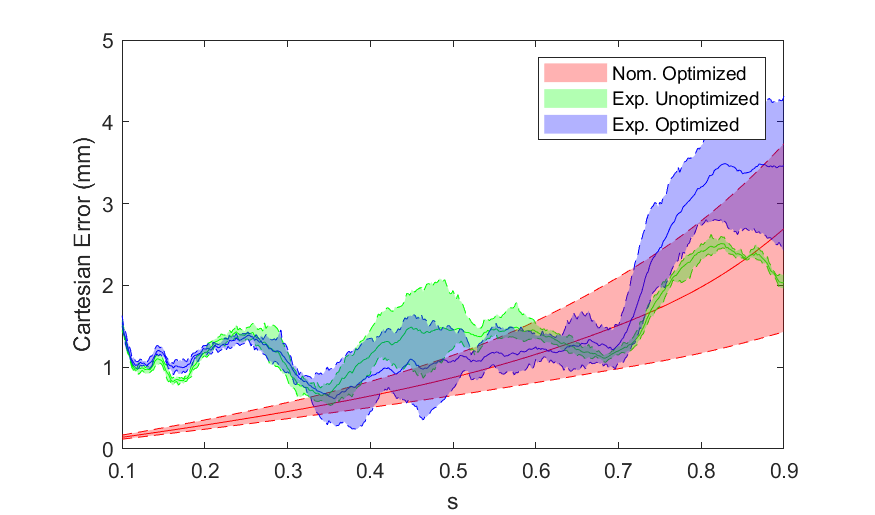} 
            \hfill 
            \includegraphics[width=8cm]{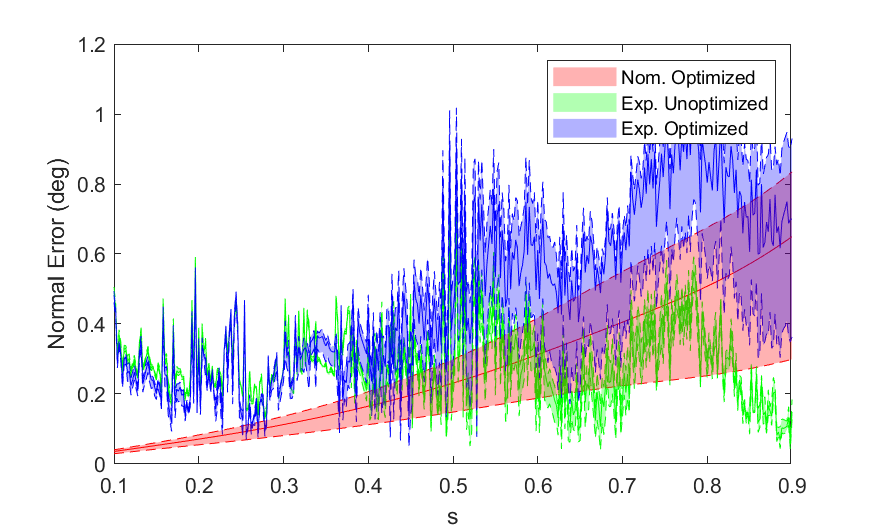}
	      \caption{Cartesian and Normal Errors for Optimized Nominal, Unoptimized and Optimized Experimental Results. The trend of having errors towards the end of the path is still visible and nominal and experimental errors are comparable in this case, albeit the latter is noisier.}
	      \label{fig:exp9}
            \end{figure}
%
%

            %

%
             In this subsection, we show results for the performance of our proposed approach for the variable velocity case, for the six different initial values of the free variable. Each experiment was terminated upon value reaching no improvement under error tolerance. The hyperparameters for each run are given as follows:
            \begin{itemize}
                \item Run 1 and 2: Step-size $3 \times 10^{-5}$. Fixed $\lambda = 50$ for the Cartesian Error, $\lambda = 0$ for the orientation.
                \item Run 3 and 4: Step-size $1.5\times10^{-4}$. Fixed $\lambda = 50$ for the Cartesian Error, $\lambda = 0$ for the orientation.
                \item Run 5: Step-size $1.5\times10^{-3}$. Fixed $\lambda = 50$ for the Cartesian Error, $\lambda = 0.01$ for the orientation.
                \item Run 6: Step-size $3\times10^{-6}$. Fixed $\lambda = 50$ for the Cartesian Error, $\lambda = 0$ for the orientation.
            \end{itemize}

            Experimental results for the constant velocity case are presented from Figures \ref{fig:exp6} to \ref{fig:exp10}. 

            Figure \ref{fig:exp6}(a) shows that the final curve time from the experimental results matches what had been expected from simulations. The same figure also illustrates unoptimized motion profiles now have much closer final times, when compared to the constant velocity case, for the same initial conditions.  Figure \ref{fig:exp1}(b) showcases the speed profile along the path. Both Simulated and Experimental final time improvement average 11.5\% for the variable velocity scenario. 
            
            Though this improvement looks smaller than the previous case at first glance, it is important to highlight that for the constant velocity, a single critical point makes the whole curve go at that minimum speed, yet with the variable velocity, one could still have a configuration that is fast everywhere else. We can observe this by looking once again at Figure \ref{fig:exp6}(b).

            We have executed both the initial and optimized curves for each experiment, and we can see that the time improvement between them matches what had been expected from simulation. The major difference between simulation and experimental results has been the cartesian and normal errors. We can see, however, that they consistent with the error already obtained from running the unoptimized curve. Figure \ref{fig:exp7}(a) and Figure \ref{fig:exp7}(b) illustrate this effect. However, while the maximum cartesian error is still under our proposed limit of 5 millimeters in Cartesian, the normal error has gone over the limit at 1.12 degree around the normal - while simulated max error had been 0.85 degree. 

            Figure \ref{fig:exp9} and Figure \ref{fig:exp10} represent the errors along the path for nominal optimized, experimental unoptimized, and experimental optimized curves. The curve spread in Figures \ref{fig:exp9}(a) and (b) represent the minimum, average, and maximum errors for each point in the path. Figure \ref{fig:exp10} show that the Cartesian curves obtained in each experiment are very close to the nominal curve.
            
            With this, we conclude that the curve obtained by the variable velocity optimization algorithm is reasonable and can be implemented in the actual robot with satisfactory results.

\section{Conclusions}\label{sec_conclusion}
In this work, we considered the problem of optimizing the joint motion of a redundant manipulator that must track a Cartesian path. Despite the non-convexity of the resulting problem, we exploited the intrinsic structure of the problem to achieve local solutions efficiently. In particular, we reformulate it as a bi-level problem where the low-level problem is convex. We prove that this formulation is equivalent to the original problem. Futhermore, we show that under certain hypothesis, this low-level subproblem can be solved in closed form. The efficient computation of the value is exploited by the high-level subproblem in finding a descent direction, {then we use a Primal-Dual optimization method to find a local minima for the problem}. We evaluated the proposed method numerically in a problem inspired by a cold spraying application, then evaluated it experimentally with a UR10e manipulator for a similar problem in the same class.

\section*{Acknowledgment}

The authors would like to thank Honglu He, Chen-lung Lu, Yunshi Wen, Agung Julius, and John T. Wen for the meaningful discussions related to the work in this paper.




\appendices


\section{Proof of Theorem \ref{thm:inner}}\label{app:PF0}
            \begin{figure}
        	\centering
        	\includegraphics[width=8cm]{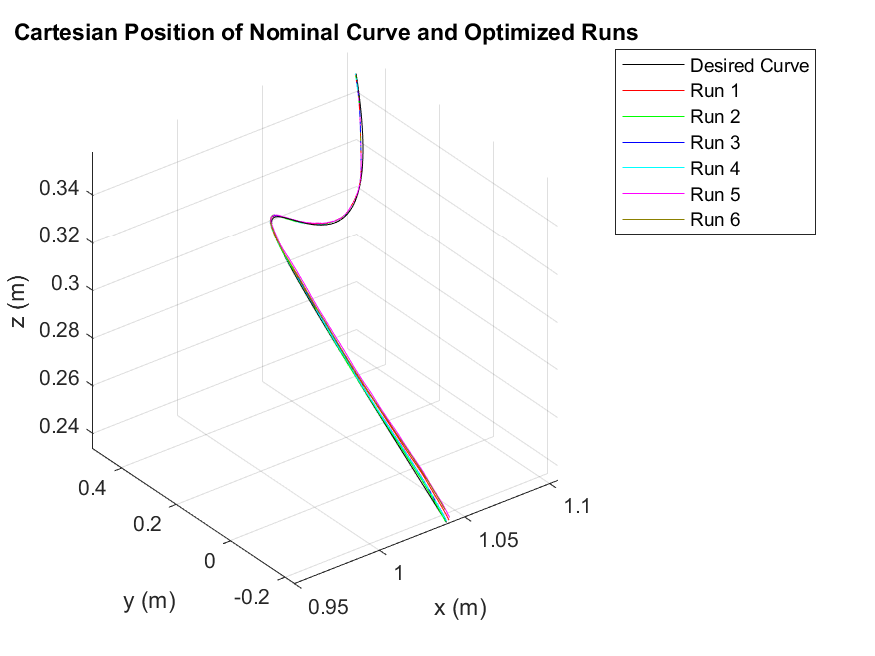} 
	      \caption{Cartesian Path for all Experiments compared to Nominal Curve. Compared experimental results at constant velocity, error is more pronounced at the end of the path.}
	      \label{fig:exp10}
            \end{figure}
Recall that 
\begin{equation}\label{eqn_aux_gradient}
 \frac{\partial \mathcal{L}_I(\theta,(\dot{s}^{2})^\star,\zeta^*)}{\partial \theta} = \frac{\partial (\zeta^*)^\top A(\theta)(\dot{s}^2)^\star}{\partial \theta}. 
 \end{equation}
 For a given $\zeta^\star \in \mathcal{Z}$, we denote by ${\zeta}_{vij}^+$ the multiplier corresponding to the positive velocity constraint for joint $j$ at point $i$. Analogously, we denote, ${\zeta}_{vij}^-$ for those multipliers corresponding to negative velocity constraints and ${\zeta}_{aij}^+$ for positive acceleration, and ${\zeta}_{aij}^-$ for negative acceleration constraints, respectively. With these considerations, rewrite \eqref{eqn_aux_gradient} as 
\begin{align}
     &(\zeta^*)^\top A(\theta)(\dot{s}^2)^\star = \sum_{i=0}^N\sum_{j=1}^n{\zeta}_{vij}^+\mbox{sign}(p_i'\theta_j)(p_i'\theta_j)^2(\dot{s}^2)^\star  \\ &- \sum_{i=0}^N\sum_{j=1}^n{\zeta}_{vij}^-\mbox{sign}(p_i'\theta_j)(p_i'\theta_j)^2(\dot{s}^2)^\star \nonumber \\ &+ \sum_{i=0}^{N-1}\sum_{j=1}^n{\zeta}_{aij}^+\left((p^{\prime\prime}_i\theta_j - \frac{p^{\prime}_i\theta_j}{2\Delta s_i})(\dot{s}^2)^\star + \frac{p^{\prime}_i\theta_j}{2\Delta s_i}(\dot{s}_{i+1}^2)^\star \right) \nonumber\\ &- \sum_{i=0}^{N-1}\sum_{j=1}^n{\zeta}_{aij}^-\left((p^{\prime\prime}_i\theta_j - \frac{p^{\prime}_i\theta_j}{2\Delta s_i})(\dot{s}^2)^\star + \frac{p^{\prime}_i\theta_j}{2\Delta s_i}(\dot{s}_{i+1}^2)^\star\right), \nonumber
\end{align}
Note that because of complementary slackness it follows that ${\zeta}_{vij}^+ > 0$, $\mbox{sign}(p_i'\theta_j) = 1$ and ${\zeta}_{vij}^+ > 0$, $\mbox{sign}(p_i'\theta_j) = -1$. This allows us to write the previous expression in a compact form
\begin{align}
    \label{eq:inner_lag_exp_app}
     &(\zeta^*)^\top A(\theta)(\dot{s}^2)^\star = \sum_{i=0}^N\sum_{j=1}^n\left({\zeta}_{vij}^+ + {\zeta}_{vij}^+\right)(p_i'\theta_j)^2(\dot{s}^2)^\star  \\  &\!\!\!+ \sum_{i=0}^{N-1}\sum_{j=1}^n\left({\zeta}_{aij}^+-{\zeta}_{aij}^+\right)\!\!\left(\!\!(p^{\prime\prime}_i\theta_j - \frac{p^{\prime}_i\theta_j}{2\Delta s_i})(\dot{s}^2)^\star + \frac{p^{\prime}_i\theta_j}{2\Delta s_i}(\dot{s}_{i+1}^2)^* \!\!\right)\!. \nonumber
\end{align}

From \eqref{eqn_aux_gradient} and computing the partial derivative of \eqref{eq:inner_lag_exp_app} with respect to $\theta_j$, the parameters of joint $j$, it follows that 
\begin{align}
     &\frac{\partial \mathcal{L}_I (\theta,(\dot{s}^2)^*,\zeta^*)}{\partial \theta_j}   =  2 \sum_{i=0}^N\left({\zeta}_{vij}^++{\zeta}_{vij}^-\right)(\dot{s}^2)^\star(p_i'\theta_j)p_i^{\prime T} \nonumber \\ &\!\!\!\quad+ \sum_{i=0}^{N-1}\left({\zeta}_{aij}^+ - {\zeta}_{aij}^-\right)\!\!\left(\!\!(p_i^{\prime\prime T} - \frac{1}{2\Delta s_i}p_i^{\prime T})(\dot{s}^2)^\star + \frac{(\dot{s}_{i+1}^2)^\star}{2\Delta s_i}p_i^{\prime T}\!\!\right)\!.
    \label{eq:der_inner_lag_app}
\end{align}
Next compute the inner product between the partial derivative in \eqref{eq:der_inner_lag_app} and $\theta_j$
\begin{align}
    \label{eq:der_inner_product_app}
     &\left(\frac{\partial \mathcal{L}_I (\theta,(\dot{s}^2)^*,\zeta^*)}{\partial \theta_j}\right)^T\theta_j   =  2 \sum_{i=0}^N\left({\zeta}_{vij}^++{\zeta}_{vij}^-\right)(\dot{s}^2)^\star(p_i'\theta_j)^2 \nonumber\\  &\!\!\! + \sum_{i=0}^{N-1}\left({\zeta}_{aij}^+-{\zeta}_{aij}^+\right)\!\!\!\left(\!\!(p_i^{\prime\prime}\theta_j - \frac{1}{2\Delta s_i}p_i^{\prime}\theta_j)(\dot{s}^2)^\star + \frac{(\dot{s}_{i+1}^2)^\star}{2\Delta s_i}p_i^{\prime}\theta_j\!\!\right)\!.
\end{align}
Note that the above expression holds for all joints. Thus, we are left to establish that the previous inner product is positive.  Observe that the terms in the first sum are all positive since ${\zeta}_{vij}^+ \geq 0$ and ${\zeta}_{vij}^- \geq 0$. We next analyze the three possible cases for the terms in the second sums. Note that either ${\zeta}_{vij}^+\geq 0$ and ${\zeta}_{vij}^-\geq 0$, ${\zeta}_{vij}^+= 0$ and ${\zeta}_{vij}^-\geq 0$ or ${\zeta}_{vij}^+= 0$ and ${\zeta}_{vij}^+= 0$ as a joint cannot attain positive and negative acceleration limits at the same time. In the latter case the term becomes zero trivially. If the positive acceleration constraint is active the term reduces to ${\zeta}_{aij}^+\ddot{\overline{q}}_{j}$. Which is positive. If the negative acceleration constraint is active, it reduces to -${\zeta}_{aij}^-\ddot{\underline{q}}_{j}$, which is also positive. This concludes the proof of the result. $\blacksquare$
%




\section{Proof of Theorem \ref{thm:inner_cte}}\label{app:PF1}

Consider the two sets of constraints defined in \eqref{opt:inner}:
\begin{equation}
     \dot{\underline{q}}_{j}\leq p'_i\theta_j\dot{s} \leq \dot{\overline{q}}_{j},\; \forall j=1,\dots,n,
     \label{eq:ctrspd}
\end{equation}
\begin{equation}
     \ddot{\underline{q}}_{j} \leq p''_i\theta_j\dot{s}^2 \leq \ddot{\overline{q}}_{j}\; \forall j=1,\dots,n. 
\label{eq:ctracc}
\end{equation}

Let us define the following value of $\dot{s}$ 
\begin{equation}
    \label{eq:sdagger}
    \dot{s}^{\dagger} \!\!=\!\!\!\!\!  \min_{\substack{j \in\{1,\dots,n\} \\ i\in\{0,\ldots,N\}}}\!\! \left\{\max\left\{\frac{\dot{\overline{q}}_{j}}{p'_i\theta_j},\frac{\dot{\underline{q}}_{j}}{p^\prime_i\theta_j}\right\},\sqrt{\max\left\{\frac{\ddot{\overline{q}}_{j}}{p^{\prime\prime}_i\theta_j},\frac{\ddot{\underline{q}}_{j}}{p{\prime\prime}_i\theta_j}\right\}}\right\}.
\end{equation}
Evaluating $\dot{s}^{-2}$ at $\dot{s}^\dagger$ yields}}
\begin{equation}
    \label{eq:pdagger}
        V^\dagger(\theta) \!\!=\!\!\!\!\!\! \max_{\substack{j \in\{1,\dots,n\} \\ i\in\{0,\ldots,N\}}} \!\!\!\left\{\!\!\left(\!\!\max\left\{\frac{p'_i\theta_j}{\dot{\overline{q}}_{j}},\frac{p^{\prime}_i\theta_j}{\dot{\underline{q}}_{j}}\right\}\!\!\right)^2\!\!\!\!\!,  \max\!\!\left\{\frac{p''_i\theta_j}{\ddot{\overline{q}}_{j}},\frac{p^{\prime\prime}_i\theta_j}{\ddot{\underline{q}}_{j}}\right\}\!\!\!\right\}.
\end{equation}
We next show that $\dot{s}^\dagger$ is a feasible point of \eqref{opt:inner}. First, the two inner max functions ensure only positive terms exist in the list for the min function. Then, a minimum of positive terms is also positive, which ensures that $\dot{s}^\dagger \geq 0$.

Note that $\dot{s}^{\dagger}$ takes one of four forms. Assume that $\dot{s}^{\dagger} = \dot{\overline{q}}_{j^*}/{p'_{i^*}\theta_{j^*}}$, where $i^*, j^*$ are particular indexes that minimize the right hand side of \eqref{eq:sdagger}. 
For any $(i, j)$ pair, if $p^\prime_i\theta_j > 0$ we show that the following holds
\begin{equation}\label{eqn_aux_proof}
    \dot{\underline{q}}_{j}\leq p'_i\theta_j\dot{s}^\dagger \leq \dot{\overline{q}}_{j},\; \forall j=1,\dots,n.
\end{equation}
From the hypothesis of the theorem $(\dot{\underline{q}}_{j} < 0)$ and the fact that $\dot{s}^\dagger>0$ and $p^\prime_i\theta_j>0$, the left hand side of \eqref{eqn_aux_proof} holds. For the right-hand side, note that
\begin{equation}
    p'_i\theta_j \dot{s}^\dagger = p'_i\theta_j\frac{\dot{\overline{q}}_{j^*}}{p'_{i^*}\theta_{j^*}} \leq \dot{\overline{q}}_{j},\; \forall j=1,\dots,n,
\end{equation}
where the inequality follows by the definition of minimum. Indeed, $\dot{\overline{q}}_{j^*}/{p'_{i^*}\theta_{j^*}} \leq \dot{\overline{q}}_{j}/{p'_i\theta_j}$. An analogous analysis holds when $p^\prime_i\theta_j < 0$.

For constraints \eqref{eq:ctracc}, first consider $p^{\prime\prime}_i\theta_j > 0$. We will prove that the following holds
\begin{equation}\label{eqn_aux_proof2}
     \ddot{\underline{q}}_{j} \leq p''_i\theta_j(\dot{s}^\dagger)^2 \leq \ddot{\overline{q}}_{j}\; \forall j=1,\dots,n.
\end{equation}
By hypothesis $(\ddot{\underline{q}}_{j}<0)$ and by the fact that $\dot{s}^\dagger > 0$ the left hand side of \eqref{eqn_aux_proof2} holds. To prove the right hand note that by the definition of minimum $\dot{\overline{q}}_{j^*}/{p'_{i^*}\theta_{j^*}} \leq \sqrt{\ddot{\overline{q}}_{j}/{p''_i\theta_j}}$. Thus
\begin{equation}
     p''_i\theta_j\left(\dot{s}^\dagger\right)= p''_i\theta_j\left(\frac{\dot{\overline{q}}_{j^*}}{p'_{i^*}\theta_{j^*}}\right)^2 \leq \ddot{\overline{q}}_{j}\; \forall j=1,\dots,n. 
\end{equation}

The result for $p^{\prime\prime}_i\theta_j < 0$ is analogous. This proves that $\dot{s}^\dagger\!\!  =\!\! \dot{\overline{q}}_{j^*}/{p'_{i^*}\theta_{j^*}}$ is feasible also for the set of constraints \eqref{eq:ctracc}.

The proof for the other three cases of $\dot{s}^\dagger$ is analogous and will be omitted for brevity. This shows that $\dot{s}^\dagger$ is a feasible point overall. Hence, $V(\theta) \leq V^\dagger(\theta)$. 

To complete the proof for the theorem, we need to show that $V(\theta) = V^\dagger(\theta)$. We argue by contradiction. Assume that $V^\dagger(\theta)> V(\theta)$. Since the cost function in \eqref{opt:inner} monotonously decreases, 
the optimal solution satisfies $\dot{s}^\star= \dot{s}^\dagger + \epsilon$, for some $\epsilon > 0$. We next show that $\dot{s}^\star$ is infeasible thus completing the proof. When $\dot{s}^\star$ has the form $\dot{s}^{\star} = \dot{\overline{q}}_{j^*}/{p'_{i^*}\theta_{j^*}} + \epsilon$. Then
\begin{equation}
    p'_{i^*}\theta_{j^*}\dot{s}^{\star} = p'_{i^*}\theta_{j^*}\left(\frac{\dot{\overline{q}}_{j^*}}{p'_{i^*}\theta_{j^*}} + \epsilon\right) = \dot{\overline{q}}_{j^*} + p'_{i^*}\theta_{j^*}\epsilon > \dot{\overline{q}}_{j^*},
\end{equation}
which means $\dot{s}^{\star}$ is infeasible. An analogous result can be made for the other three forms of $\dot{s}^{\dagger}$. This concludes the proof of the result.

\section*{References}
\vspace{-15pt}
\bibliographystyle{IEEEtran}
\bibliography{cdc2023bib}             

\begin{IEEEbiography}[{\includegraphics[width=1in,height=1.25in,clip,keepaspectratio]{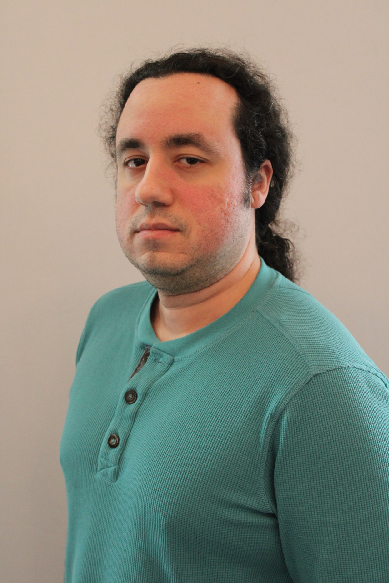}}]
{Jonathan Fried} 
received a B.Sc. degree in control and automation engineering from Universidade Federal do Rio de Janeiro, Rio de Janeiro, Brazil in 2015, M.Sc. in electrical engineering from the same university in 2019, and has been
working toward the Ph.D. degree in the Department of Electrical Computer and Systems Engineering at Rensselaer Polytechnic Institute since January 2022. His research interests include optimization, classic control, robotics, and reinforcement learning.
\end{IEEEbiography}

\vspace{-0.5cm}

\vspace{-0.5cm}

\begin{IEEEbiography}[{\includegraphics[width=1in,height=1.25in,clip,keepaspectratio]{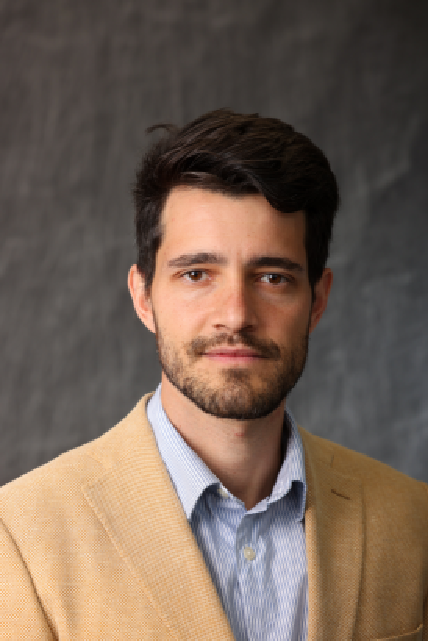}}]
{Santiago Paternain} received the B.Sc. degree in electrical engineering from Universidad de
la República Oriental del Uruguay, Montevideo, Uruguay in 2012, the M.Sc. in Statistics from the Wharton School in 2018 and the Ph.D. in Electrical and Systems Engineering from the Department of Electrical and Systems Engineering, the University of Pennsylvania in 2018. He is currently an Assistant Professor in the Department of Electrical Computer and Systems Engineering at Rensselaer Polytechnic Institute. Prior to joining Rensselaer, Dr.
Paternain was a postdoctoral Researcher at the University of Pennsylvania. His research interests lie at the intersection of machine learning and control of dynamical systems. Dr. Paternain was the recipient of the 2017 CDC Best Student Paper Award and the 2019 Joseph and Rosaline Wolfe Best Doctoral Dissertation Award from the Electrical and Systems Engineering Department at the University of Pennsylvania.
\end{IEEEbiography}

\end{document}